\documentclass{amsart}

\usepackage{amsmath,amssymb,amsthm}

\usepackage{tikz}
\usetikzlibrary{calc,intersections,through,backgrounds}
\usepackage[latin1]{inputenc}

\usepackage{caption}
\usepackage{graphicx}
\usepackage{graphics}

\newtheorem*{thm}{Theorem}
\newtheorem*{lemma}{Lemma}

\begin{document}

\title[]{t-SNE, Forceful Colorings and Mean Field Limits}

\author[]{Yulan Zhang}
\address[]{Yale University, New Haven, CT 06511, USA} \email{yulan.zhang@yale.edu}

\author[]{Stefan Steinerberger}
\address[]{Department of Mathematics, University of Washington, Seattle, WA 98195, USA} \email{steinerb@uw.edu}

\thanks{Y. Z. was partially supported by a grant of the Poorvu Family. S.S. was partially supported by the NSF (DMS-1763179) and the Alfred P. Sloan Foundation.}

\begin{abstract} t-SNE is one of the most commonly used force-based nonlinear dimensionality reduction methods. This paper has two contributions: the first is \textit{forceful colorings}, an
idea that is also applicable to other force-based methods (UMAP, ForceAtlas2,\dots). In every equilibrium, the attractive and repulsive forces acting on a particle cancel out:
however, both the size and the direction of the attractive (or repulsive) forces acting on a particle are related to its properties: the force vector can serve as an additional feature. Secondly, we analyze the case of t-SNE acting
on a single homogeneous cluster (modeled by affinities coming from the adjacency matrix of a random $k-$regular graph); we derive a mean-field model that leads to 
interesting questions in classical calculus of variations. The model predicts that, in the limit, the t-SNE embedding of a single perfectly homogeneous cluster is not a point but 
a thin annulus of diameter $\sim k^{-1/4} n^{-1/4}$. This is supported by numerical results. The mean field ansatz extends to other force-based dimensionality reduction methods.
\end{abstract}
\maketitle

\section{Introduction}
t-distributed Stochastic Neighborhood Embedding (t-SNE) is a well-known nonlinear dimensionality reduction technique with applications in many fields. It is frequently used to generate two- or three-dimensional visualizations of high dimensional datasets, often for the purpose of visualizing and identifying clusters. 
\vspace{-10pt}

\begin{center}
\begin{figure}[h!]
\begin{tikzpicture}
\node at (0,0) {\includegraphics[width=0.43\textwidth]{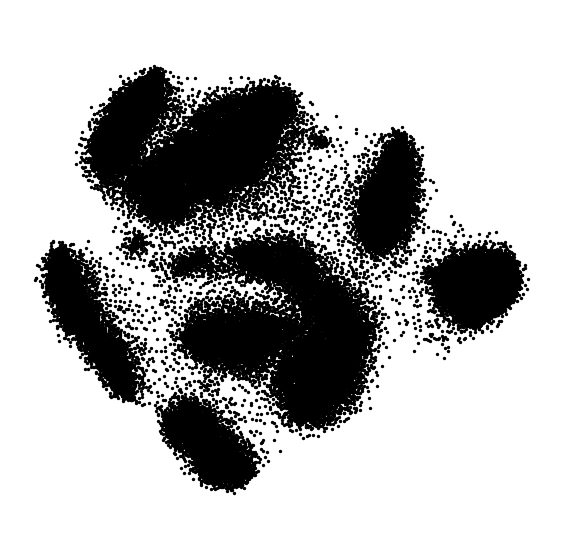}};
\node at (5,0) {\includegraphics[width=0.43\textwidth]{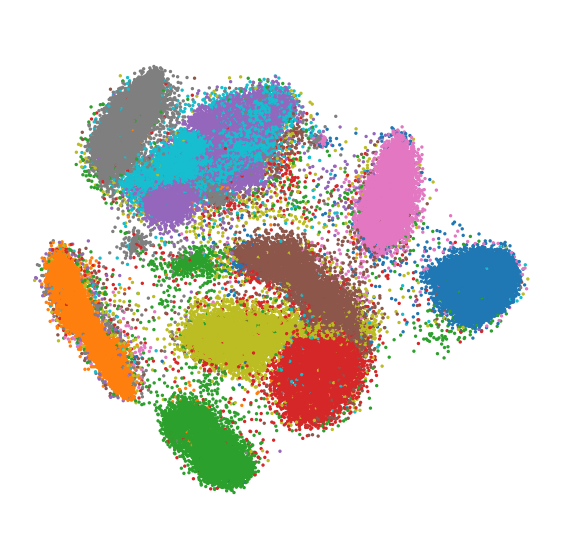}};
\end{tikzpicture}
\vspace{-10pt}
\caption{t-SNE embedding of MNIST (left) with ground truth coloring (right).}
\label{fig:mnist}
\end{figure}
\end{center}

We describe t-SNE as a \textit{force-based} method because it generates embeddings by balancing attractive and repulsive forces between data samples. These forces are determined by comparing the neighborhood structure of the input data to that of the output. Other well-known force-based methods include Laplacian eigenmaps \cite{bel, coif}, ForceAtlas2 \cite{jac}, LargeVis \cite{largevis}, and UMAP \cite{umap}. \\
% yulan: did not include SNE \cite{hint}

\begin{center}
\begin{figure}[h!]
\begin{tikzpicture}
\node at (0,0) {\includegraphics[width=0.9\textwidth]{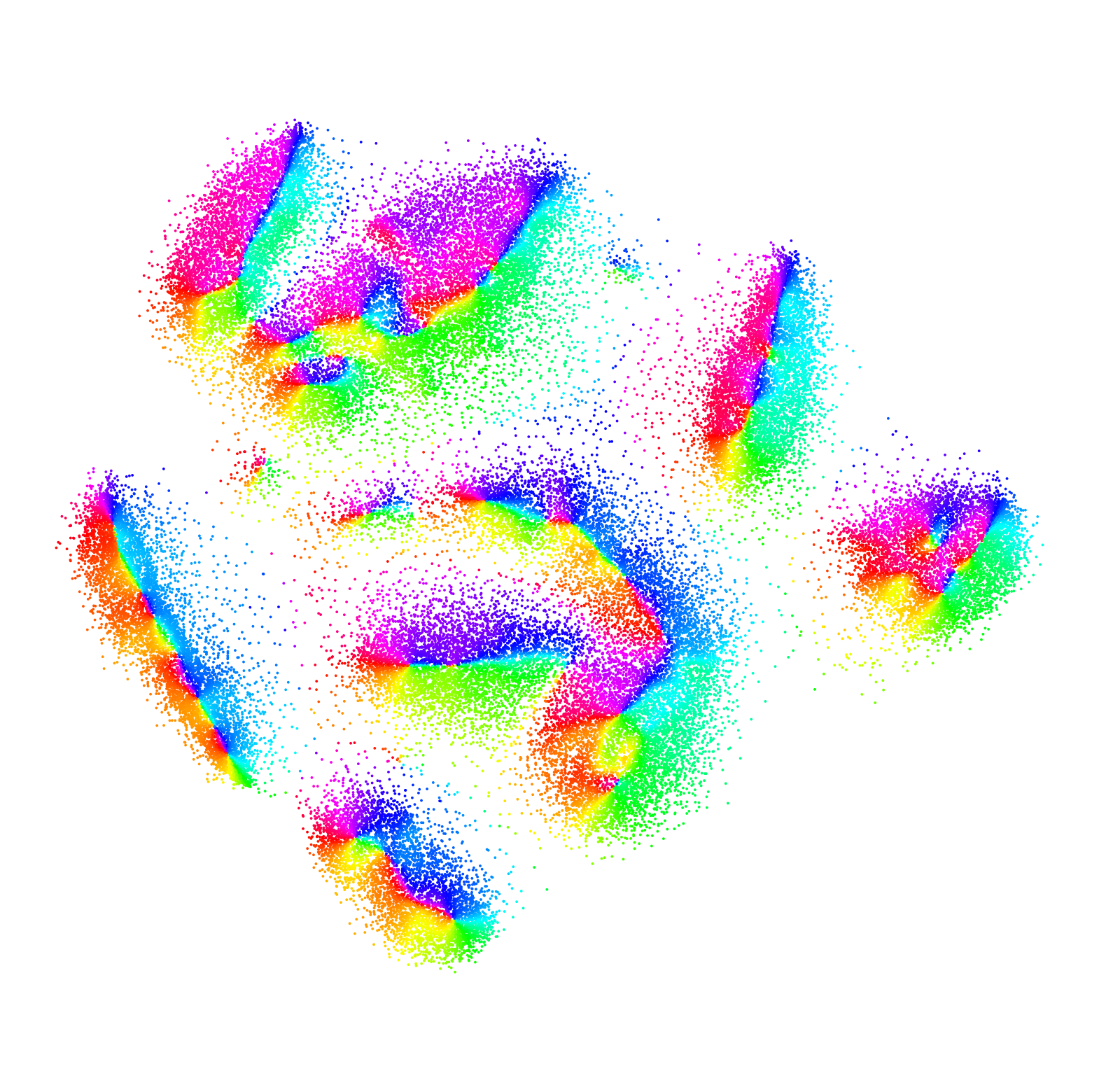}};
\node at (3.75,-3.2) {\includegraphics[width=0.2\textwidth]{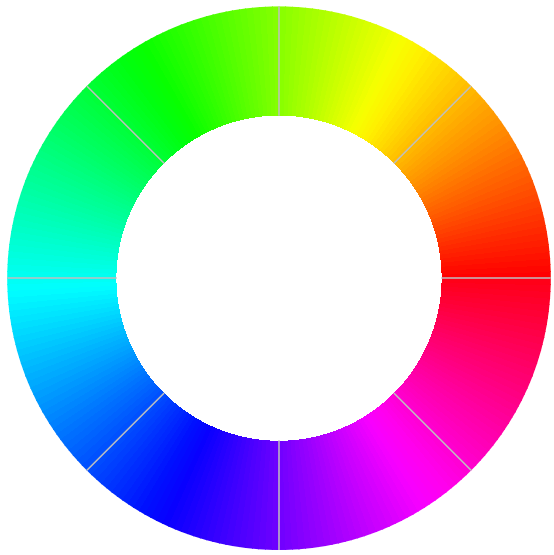}};
\end{tikzpicture}
\vspace{-10pt}
\caption{Coloring of t-SNE embedding by force direction. We propose that using forces in the equilibrium embedding as features can provide additional information for (sub-)cluster identification. The wheel in the bottom right identifies colors with directions.}
\label{fig:mnist_directions}
\end{figure}
\end{center}

Although t-SNE is widely used in applications, there is currently little theory to explain how it works. The algorithm can have profoundly different outputs with different choices of parameters, and it is well known that it simply does not work well with certain types of data, such as manifolds \cite{george1, wattenberg}. Identifying when t-SNE results are meaningful and how they should be interpreted is thus an important open question for its practical use. \\

\textbf{Existing results.} 
Linderman \& Steinerberger \cite{george1} interpreted t-SNE as a dynamical multi-particle system that obeys certain ellipticity conditions. This approach was further developed by Arora, Hu \& Kothari \cite{arora}. These results show, roughly, that if the underlying data $\left\{x_1, \dots, x_n\right\} \subset \mathbb{R}^d$ is strongly clustered, then t-SNE will recover the clustering. The results also explain, more qualitatively than quantitatively, the underlying mechanism by which this occurs. One goal of this paper is to introduce a new approach for obtaining quantitative predictions of t-SNE results. \\

Since t-SNE is highly popular, there are many experimental studies and guides for selecting parameters and validating results. We especially highlight two recent studies by Kobak \& Berens \cite{kobak} and Wang, Huang, Rudin \& Shaposhnik \cite{wang}. We also point out the study by B\"ohm, Behrens \& Kobak \cite{bohm}, which shows that force-based methods lie on an attraction-repulsion spectrum and can be empirically recovered by tuning the forces used to create the embedding. We believe that this idea is a very promising step towards a unified theory of these algorithms. \\

\textbf{Outline of the paper.} We discuss two (independent) new ideas.
\begin{enumerate}
\item \textbf{Forceful Colorings.} We propose using the attractive and repulsive \textit{forces} used to generate the t-SNE embedding as features. Naturally, when the embedding reaches equilibrium, these force vectors cancel; however, we find that either one of the two can be used as an additional classifier that carries a lot of information. This idea can be applied to any force-based technique and will be explained in greater detail in \S 2.
\item \textbf{Mean Field Limits.} We present a new approach for obtaining quantitative predictions on the behavior of minimizers of the t-SNE energy (cost function). The main idea is to base the analysis on assumptions about the input similarities $p_{ij}$ rather than the data $\left\{x_1, \dots, x_n\right\}$. In particular, we set $p_{ij}$ as the adjacency matrix of a random graph. For suitable graph models, such as Erd\H{o}s-Renyi or random $k$-regular graphs, a stochastic regularization phenomenon allows us to simplify and rewrite the t-SNE cost as a fairly classical calculus of variations problem. We solve the problem for random $k$-regular graphs and come to an interesting conclusion: the mean field limit predicts that the energy minimizer of a $k-$regular random graph is, asymptotically, given by an \textit{annulus}. This result is interesting in its own right, but it also highlights how little we actually know about the t-SNE energy. These results are described in \S 3 and derived in \S 4. 
\end{enumerate}

\section{Forceful Colorings}
\subsection{Force-based methods.} This section presents a simple new idea which may prove to be useful for applications of force-based embedding methods. We begin by describing the logic behind force-based methods in a unified way. A more complete description of t-SNE specifically is given in \S 4.1. 
Most force-based dimensionality reduction techniques work by minimizing some notion of energy $E$ for the output embedding $\mathcal{Y} = \left\{y_1, \dots, y_n\right\} \subset \mathbb{R}^s$. Letting $\mathcal{X} = \left\{x_1, \dots, x_n \right\} \subset \mathbb{R}^d$ be our input dataset, we initialize $\mathcal{Y}$ and apply an iterative method on the coordinates to minimize $E$. Each step of the optimization can usually be interpreted as an interaction between attractive and repulsive forces that moves the particle system $\mathcal{Y}$ towards a locally optimal configuration. 
\begin{center}
\begin{figure}[h!]
\begin{tikzpicture}
\node at (0,0) {\includegraphics[width=0.38\textwidth]{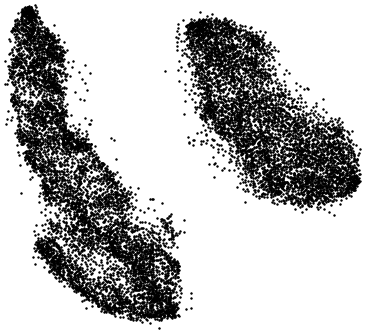}};
\node at (6,0) {\includegraphics[width=0.4\textwidth]{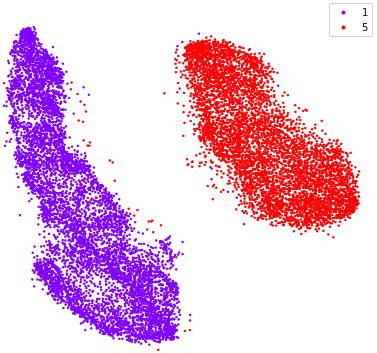}};
\node at (3.5,-6) {\includegraphics[width=0.8\textwidth]{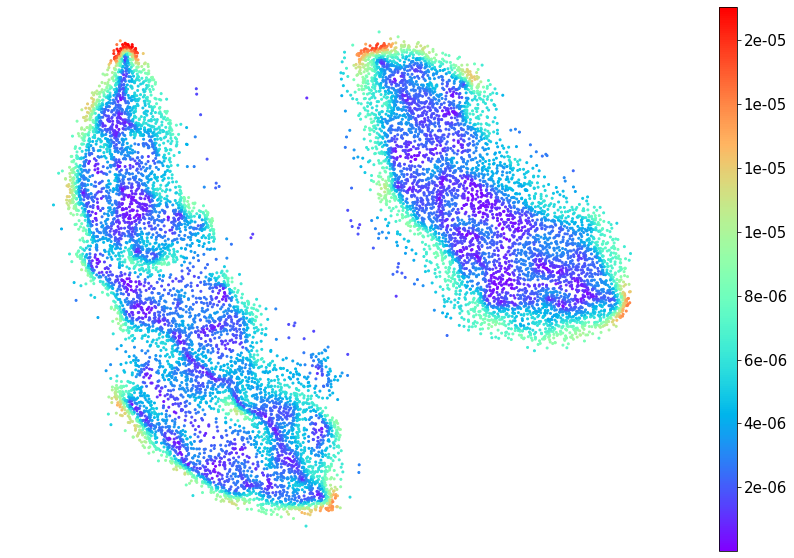}};
\end{tikzpicture}
\vspace{-10pt}
\caption{t-SNE embedding of the digits 1 and 5 from MNIST (top left) with ground truth labels (top right). Coloring by the magnitude of attractive forces on a point (bottom) hints at substructures within the clusters.}
\label{fig:magnitude}
\end{figure}
\end{center}

For t-SNE specifically, we use gradient descent to optimize the energy functional
$$
E(y_1, \dots, y_n) = \sum_{i \neq j} p_{ij} \log\left(\frac{p_{ij}}{q_{ij}}\right)
$$
Here $p_{ij}$ represents pairwise similarities in the input space $\mathbb{R}^d$ and $q_{ij}$ represents pairwise similarities in the output space $\mathbb{R}^s$. $E$ is minimized when $p_{ij}$ and $q_{ij}$ have the same distribution. We update each $y_i$ using the negative gradient:
$$ 
-\frac{\partial E}{\partial y_i} = 4\sum_{j \neq i} p_{ij} q_{ij} Z (y_j - y_i) - 4\sum_{j \neq i} q_{ij}^2 Z (y_j - y_i),
$$
where $Z$ is a normalization factor for $q_{ij}$ calculated from $\mathcal{Y}$. The first term is an attractive force that moves $y_i$ towards points $y_j$ for which $p_{ij}$ is large. These points correspond to samples $x_j$ which are close to sample $x_i$ in the input data. The second term is a repulsive force that moves $y_i$ away from points $y_j$ for which it is too close. This prevents the formation of degenerate clusters. The net effect, hopefully, is that attraction dominates for pairs of points which are nearby in $\mathcal{X}$ while repulsion dominates for pairs of points which are distant, so that the final embedding $\mathcal{Y}$ preserves the neighborhood relations of the input.

\subsection{Forceful Colorings.} We reach a local minimum of the t-SNE energy functional when the attractive and repulsive forces on each $y_i$ cancel, i.e. 
$$ 
\frac{\partial E}{\partial y_i} = 0 \qquad \qquad \forall~1 \leq i \leq n
$$
Though the net force on each point is $0$, the magnitudes of the attraction and repulsion (generally) do not vanish. The main insight is that these forces actually 
%\vspace{-10pt}
\begin{center}
\begin{figure}[h!]
\begin{tikzpicture}
\node at (2.5,-2.3) {\includegraphics[width=0.17\textwidth]{wheel}};
\node at (0,0) {\includegraphics[width=0.7\textwidth]{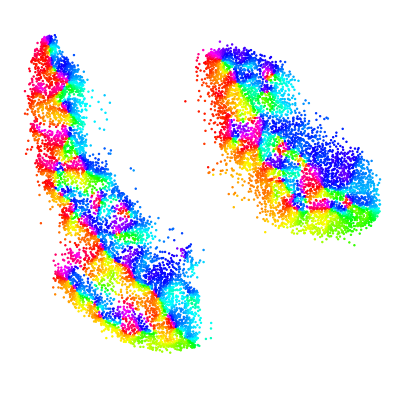}};
\end{tikzpicture}
\vspace{-10pt}
\caption{t-SNE embedding of MNIST 1 and 5 (see also Fig. \ref{fig:magnitude}) colored by direction of the attractive forces acting on a point. The wheel identifies colors with directions. The forceful coloring reveals rich substructures.}
\label{fig:dir}
\end{figure}
\end{center}
contain substantial information on the embedding structure while being easy to calculate. In fact, they are computed as part of each gradient descent step.
\begin{quote}
\textbf{Main idea}\textbf{.} 
The attractive (or, equivalently, repulsive) forces on a particle organize clusters into force sinks (sources) which can be used to identify meaningful substructures in the data.
\end{quote}

This principle is based on empirical observations. We have not found it stated elsewhere in the literature, and we believe it to be possibly quite useful. A priori, a generic t-SNE embedding can be challenging to interpret, as it is not always clear how exactly to separate clusters. In Fig. \ref{fig:mnist}, for example, we see that it is impossible to distinguish the purple and light blue clusters, representing $4$ and $9$ respectively, based on the raw output. When we color the embedding by directions, however, we see the emergence of patterns that roughly correspond to the underlying ground truth (Fig. \ref{fig:mnist_directions}). We observe a similar phenomenon for the brown, yellow, and red clusters (representing $5$, $8$, and $3$).

\begin{center}
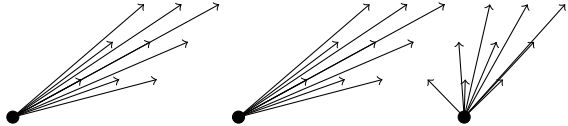
\begin{figure}[h!]
\begin{tikzpicture}
\filldraw (0,0) circle (0.08cm);
\filldraw (3,0) circle (0.08cm);
\filldraw (6,0) circle (0.08cm);
  \foreach \x in {1,...,3}
    \foreach \y in {1,...,3}
      {
        \draw [->] (0,0) -- (0.5*\x+0.4142*\y, 0.5*\y);
      \draw [->] (3,0) -- (3+0.5*\x+0.4142*\y, 0.5*\y);
      \draw [->] (6,0) -- (4.6+0.5*\x+0.4142*\y, 0.5*\y);
      }
\end{tikzpicture}
\caption{Another interpretation of forceful colorings: since t-SNE preserves neighborhood structure, we expect that points which are similar in the input data will be subject to similar forces. On the other hand, nonhomogenous force vectors may indicate that the points are quite different despite being close in the embedding.}
\label{fig:rep}
\end{figure}
\end{center}

\begin{center}
\begin{figure}[b!]
\begin{tikzpicture}
\node at (0,-2.5) {\includegraphics[width=0.25\textwidth]{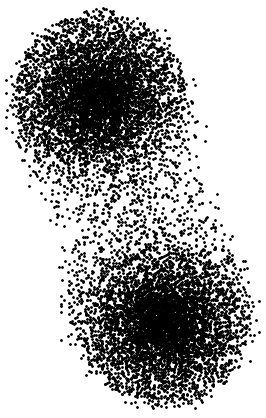}};
\node at (3,-2.5) {\includegraphics[width=0.25\textwidth]{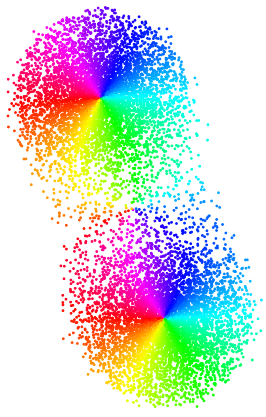}};
\node at (6.5,-2.5) {\includegraphics[width=0.35\textwidth]{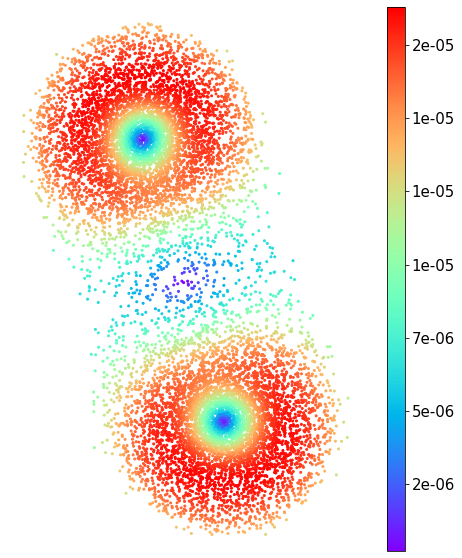}};
\end{tikzpicture}
\vspace{-10pt}
\caption{t-SNE embedding of two Gaussian clusters (left), the attractive forces (middle) and size of these forces (right). There is an ambiguous region in the middle, but it is possible to discern cluster identity from the direction of the attractive forces.}
\label{fig:gauss}
\end{figure}
\end{center}
\vspace{-20pt}

We also hypothesize that the force vectors can be used to measure the local homogeneity of the data. If two points $x_i$ and $x_j$ are similar in the original dataset, then they likely have similar affinities $p_{k\ell}$. As a result, we can expect that (1) $y_i$ and $y_j$ will be nearby in the final t-SNE embedding and (2) attractive forces on $y_i$ and $y_j$ will have similar magnitudes and directions. On the other hand, if the forces on nearby embedded points are highly dissimilar, they may represent dramatically different samples in the original dataset (Fig. \ref{fig:rep} and Fig. \ref{fig:gauss}). 

\subsection{Magnitude and Direction.} We found it interesting to consider the magnitude and direction of the attraction (repulsion) vectors in isolation. As examples, we plotted embeddings of the digits $1$ and $5$ from MNIST (Fig. \ref{fig:magnitude}, \ref{fig:dir}) and two high-dimensional Gaussian clusters (Fig. \ref{fig:gauss}). For both datasets, we observe that the forces are generally stronger on the edges of a cluster. This is not surprising since for an embedded point $y_i$ near a cluster boundary, the points $y_j$ with high input similarity $p_{ij}$ must lie in a halfspace about $y_i$, which limits the vector cancellation of attractive forces. However, the magnitude coloring effectively illuminates the structure of the cluster's interior. In particular, we see the emergence of connected regions separated from other regions by a dramatic change of force. These internal regions become more clear when we plot the \textit{direction} of the vector. 

\begin{center}
\begin{figure}[h!]
\begin{tikzpicture}
\node at (0,0) {\includegraphics[width=0.3\textwidth]{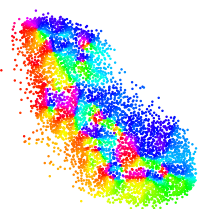}};
\draw [ultra thick] (0.3,0) -- (0.3,1.8) -- (-1.8, 1.8) -- (-1.8, 0) -- (0.3, 0);
\node at (5,0) {\includegraphics[width=0.6\textwidth]{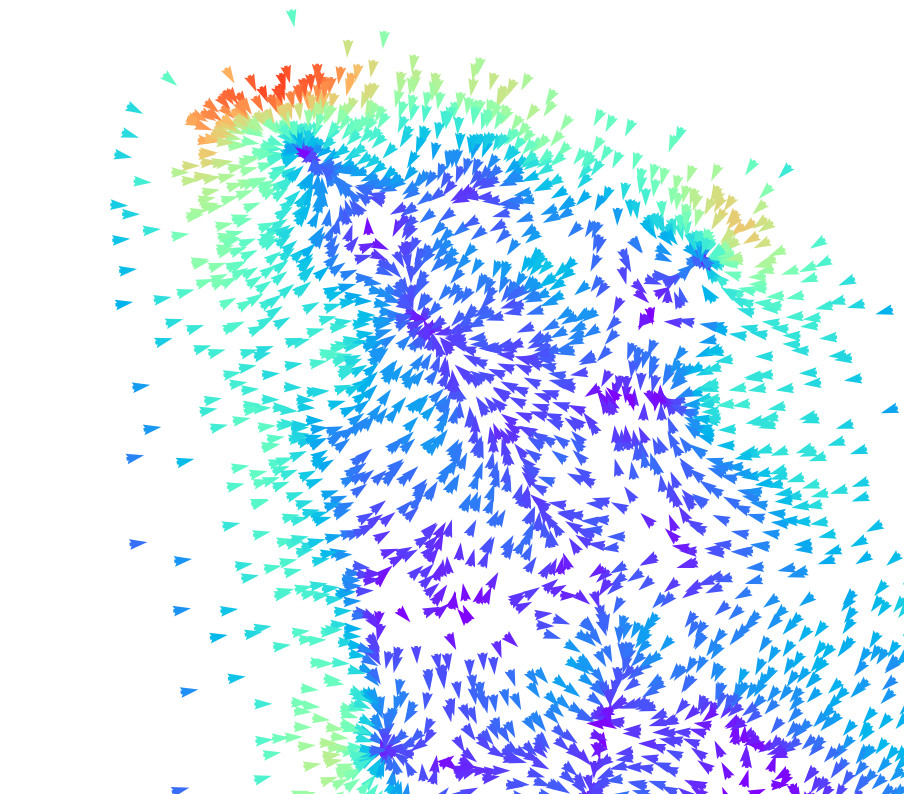}};
\end{tikzpicture}
\vspace{-10pt}
\caption{The same t-SNE embedding of MNIST 1 and 5 as above. Zooming into a tiny region (on the left) and drawing the forces as a vector field (colored by magnitude) reveals remarkable inner structure that can serve as additional feature.}
\label{fig:vector}
\end{figure}
\end{center}
\vspace{-10pt}
Naturally, we would like to use this information to refine t-SNE and other related methods. This could be done in many ways. For example, we can:
\begin{enumerate}
\item \label{enum:subcluster} use the vector field generated by the attractive forces to partition clusters into sub-clusters and/or refine cluster associations.
\item \label{enum:homogeneity} use the overall behavior of the force vector within a cluster as a measure of the cluster homogeneity.
\item \label{enum:earlyform} use force vectors for the identification of `barrier clusters' that formed too rapidly (see Fig. \ref{fig:mnist} and below for details).
\item compare the magnitude of the force vector acting on a point across multiple independent runs of t-SNE.
\item only use information from runs where the attractive forces acting on a specific particle end up being small for proper group identification.
\end{enumerate}
We illustrate (\ref{enum:subcluster}) using the MNIST embedding of $1$ and $5$. Focusing on the cluster of $5$'s, we observe that the attraction vector field contains several `sinks' -- regions where forces converge towards a single point (Fig. \ref{fig:vector}). We identified three potential subclusters using these sinks, and checked their coherence by computing their average image (Fig. \ref{fig:subclusters}). Since the images are sharp, most of the pictures in the cluster are similar to the mean. Moreover, the means themselves appear to represent different handwriting styles. For instance, digits in cluster $1$ have the most severe slant, while digits in $3$ have the most pronounced loop. This indicates that force vector fields can be useful for identifying subfeatures in a cluster's interior.
%yulan: acknowledge that handwriting subfeatures are not clearly defined, would be more convincing if this analysis holds on a dataset with a more obvious substructure
%yulan: open question: we may expect that kNN are uniformly distributed about points in the cluster interior. in this case there is some possibility that the magnitude/direction is somewhat random due to cancellations. to what degree is this true? the organization of the vector directions suggest that this may not be a huge issue.

Idea (\ref{enum:homogeneity}) was inspired by Fig. \ref{fig:rep} and by the fact that the Gaussian clusters (Fig. \ref{fig:gauss}) contain a single sink while the MNIST clusters (Fig. \ref{fig:vector}) have a more turbulent vector field.
We finally comment on (\ref{enum:earlyform}). Sometimes, during the t-SNE gradient descent, data of the same type simultaneously starts forming clusters in two different regions in space. One would assume a priori that these two clusters would then move towards each other and merge into a larger cluster. However, it is sometimes possible that other clusters have formed between the two and now act as a barrier. In Fig. \ref{fig:mnist}, this is the reason purple and light blue ($4$ and $9$) are so deeply intertwined: the purple cluster would tend to move towards each other if embedded in isolation, but they are obstructed by the light blue barrier and vice versa. Naturally, this type of behavior would show up in the forceful colorings, which raises the interesting question how to identify and the circumvent this. Again, an abundance of ideas comes to mind, e.g. `teleporting' uniform clusters towards their force direction, temporarily increasing/decreasing the attractive/repulsive forces in that area, etc. 

\begin{center}
\begin{figure}[h!]
\begin{tikzpicture}
\node at (3,0) {\includegraphics[width=0.6\textwidth]{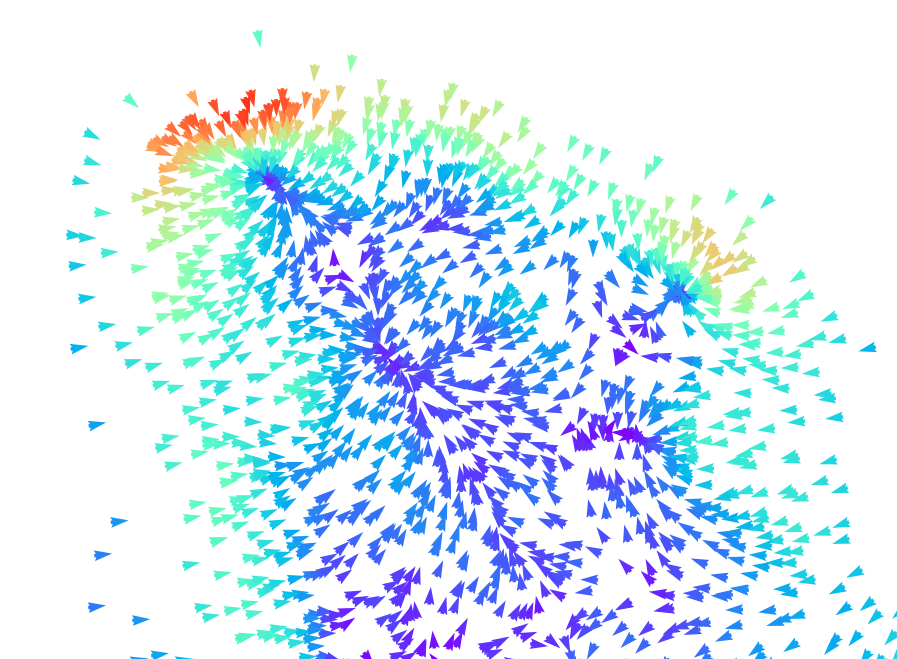}};
\draw [ultra thick] (2.3,0.6) -- (2.3, 2) -- (0.4, 2) -- (0.4, 0.6) -- (2.3, 0.6);
\draw [ultra thick] (4.1,-2.5) -- (4.1, 0.7) -- (1.5, 0.7) -- (1.5, -2.5) -- (4.1,-2.5);
\draw [ultra thick] (5.4,-0.2) -- (5.4, 1) -- (4.4, 1) -- (4.4, -0.2) -- (5.4,-0.2);
\node at (2.5, 2.1) {1};
\node at (1.2, -2.4) {2};
\node at (5.6, 1.1) {3};
\node at (-1.5,2) {\includegraphics[width=0.15\textwidth]{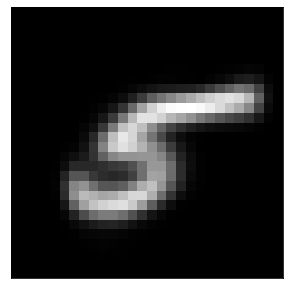}};
\node at (-1.5,0) {\includegraphics[width=0.15\textwidth]{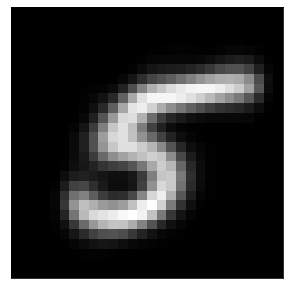}};
\node at (-1.5,-2) {\includegraphics[width=0.15\textwidth]{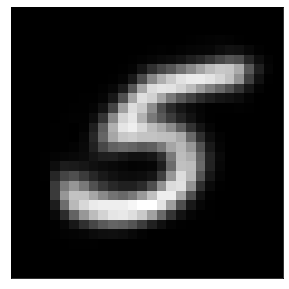}};
\node at (-2.6, 2) {1};
\node at (-2.6, 0) {2};
\node at (-2.6, -2) {3};
\end{tikzpicture}
\vspace{-10pt}
\caption{Subclusters identified using the vector field in Figure \ref{fig:vector}. The three subclusters contains roughly $335$, $563$, and $170$ samples respectively. The left hand side shows the mean MNIST image of each subcluster.}
\label{fig:subclusters}
\end{figure}
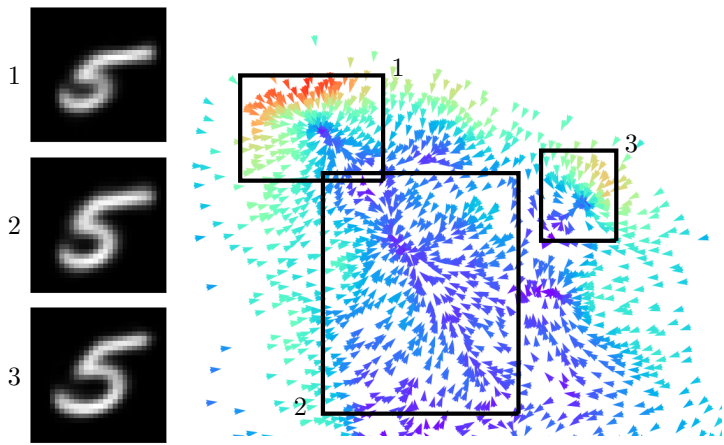
\end{center}
\vspace{-10pt}{}

We see that incorporating forces as a feature leads to many possible adaptations and variations on t-SNE and other force-based nonlinear dimensionality reduction methods. Investigating when these ideas are useful and how to best implement them seems like a very interesting avenue of future research.

\section{Mean Field Limit for t-SNE}
\subsection{Motivation.} This section describes purely theoretical work on t-SNE. Our goal was to find a simple setting in which the t-SNE functional can be studied using rigorous quantitative methods. We study the embedding of a single homogeneous cluster and emphasize 
\begin{enumerate}
\item that the underlying approach extends to more complicated settings (see \S 3.5). Such extensions lead to more complicated problems in calculus of variations that may be interesting in their own right.
\item that the underlying approach also extends to other attraction-repulsion based methods. Indeed, a similar type of analysis should be possible for many of the methods discussed in \cite{bohm, wang}.
\end{enumerate}

One reason there is so little theoretical work on t-SNE is the complexity of the setup: we are given a set of points $\mathcal{X} = \left\{x_1, \dots, x_n\right\} \subset \mathbb{R}^d$. For each pair $x_i$ and $x_j$, we define a measure of affinity $p_{ij}$. These affinities then fuel a dynamical system on $n$ particles $\mathcal{Y} = \left\{y_1, \dots, y_n \right\} \subset \mathbb{R}^s$ that determines the embedding. Each of these objects is already nontrivial on its own. The two existing theoretical approaches \cite{arora, george1} assume that the $p_{ij}$ are strongly clustered in order to deduce information about the dynamical system. Showing that t-SNE preserves pre-existing cluster structure amounts to a \textit{soft} analysis of the t-SNE mechanism. In contrast, we aim to present the first \textit{hard} analysis by making explicit quantitative statements about the output. This analysis will involve classical techniques from the calculus of variations and leads to interesting problems. It also extends to more complicated settings (see \S 3.5 for details).

\begin{center}
\begin{figure}[h!]
\begin{tikzpicture}
\node at (-0.2,0) {\includegraphics[width=0.32\textwidth]{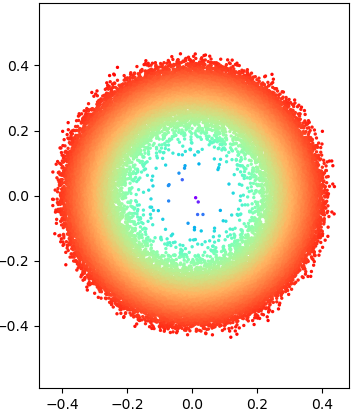}};
\node at (4,0) {\includegraphics[width=0.32\textwidth]{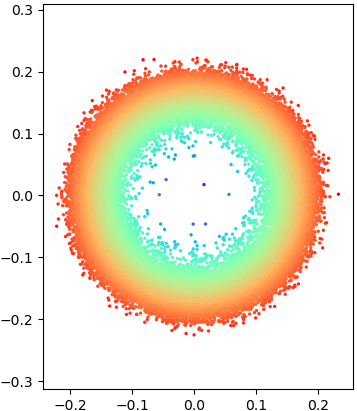}};
\node at (8.2,0) {\includegraphics[width=0.32\textwidth]{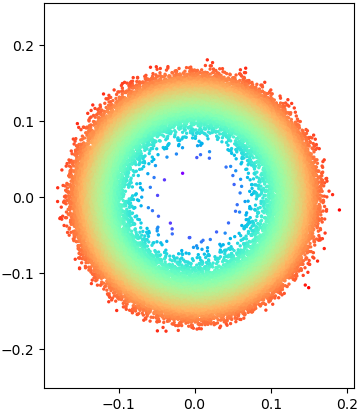}};
\end{tikzpicture}
\caption{An embedding of a $k-$regular graph on 40000 vertices: $k=40$ (left), $k=400$ (middle) and $k=4000$ (right). The mean field model predicts the diameter of the ring to scale as $\sim k^{-1/4} n^{-1/4}$. The emerging ring structure is \textit{not} reflective of any underlying circular structure in the data and purely an artifact of the variational structure of the t-SNE functional.}
\label{fig:ring}
\end{figure}
\end{center}
\vspace{-20pt}

\subsection{Random Regular Graphs and their Mean Fields.}
In the t-SNE algorithm, the input data set $\mathcal{X} \subset \mathbb{R}^d$ does not directly enter into the computation of the final output $\mathcal{Y}$. Rather, it is the affinities $p_{ij}$ on $\mathcal{X}$ which are used to generate $\mathcal{Y}$. We will argue that for the purpose of developing rigorous mathematical theory, it may be advantageous not to study t-SNE under some assumptions on $\mathcal{X}$ but to start with the setting where one only poses assumptions on the $p_{ij}$.
\begin{quote}
\textbf{Main Idea.} Instead of trying to impose structure on the original points $\left\{x_1, \dots, x_n\right\}$, a rigorous analysis of t-SNE should first address the case where the affinities $p_{ij}$ are structured. In particular, when the $p_{ij}$ are taken as the entries of an adjacency matrix of certain types of random graphs, there is a stochastic regularization phenomenon that simplifies the structure of the t-SNE energy.
\end{quote}
We tried to understand the implications of this idea in the very simplest case: embedding a single cluster in two dimensions.
There are at least two canonical models of what a perfectly homogeneous cluster could look like: (1) a random $k-$regular graph and (2) the Erd\H{o}s-Renyi random graph $G(n,p)$. We will see that with regards to an effective mean-field limit, both models behave somewhat similarly. A more refined analysis shows that one of the t-SNE energy terms has a larger variance under the Erd\H{o}s-Renyi model. This is also confirmed by numerical experiments: pictures like Fig. \ref{fig:ring} are easy to produce for $k-$regular graphs but it is not possible to get equally clear ring structures for the Erd\H{o}s-Renyi model (perhaps all that is required is a larger number of points but it could conceivably also be a difference in the actual variational structure).  

\begin{center}
\begin{figure}[h!]
\begin{tikzpicture}[scale=4]
  \foreach \x in {0,...,20}
    \foreach \y in {0,...,20}
      {
        \filldraw (0.05*\x+ 0.03*rand,0.05*\y+0.03*rand) circle (0.01cm);
      }
  \draw [thick] (-0.1,-0.1) -- (0.2, -0.1) -- (0.2, 0.2) -- (-0.1, 0.2) -- (-0.1,-0.1);    
    \draw [thick] (0.7,0.7) -- (1.1, 0.7) -- (1.1, 1.1) -- (0.7, 1.1) -- (0.7,0.7);    
\node at (-0.2, 0) {$A$};    
\node at (1.2, 0.9) {$B$};    
\end{tikzpicture}
\caption{Suppose the points are the final embedding of the vertices of an Erd\H{o}s-Renyi or a random $k-$regular graph: the number of edges running between the vertices in $A$ and the vertices in $B$ is under control and depends (up to a small error) only on the number of vertices in $A$ and $B$ \textit{independently of the embedding}.}
\label{fig:reg}
\end{figure}
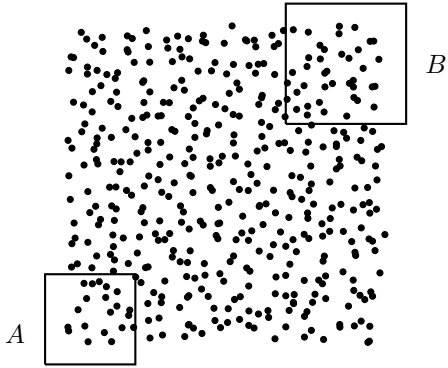
\end{center}
\vspace{-10pt}

Our derivation will initially not distinguish between the Erd\H{o}s-Renyi model and the model of a $k-$regular random graph (with $k \sim p \cdot n$). In \S \ref{sec:erdmean}, we will discuss the arising quantities for the Erd\H{o}s-Renyi model and describe the crucial variance term which disappears in the random $k-$regular model. Throughout the rest of the paper, we will then refine our argument for random $k-$regular graphs.
 The Erd\H{o}s-Renyi random graph $G(n,p)$ is a graph on $n$ vertices where any pair of vertices is connected with likelihood $0 < p < 1$, where $0 < p < 1$ is fixed and we let $n$ become large (this is the dense setting). For such a random graph, we set
$$ p_{ij} = \begin{cases} 1 \qquad &\mbox{if}~i \sim_{E} j \\
0 \qquad &\mbox{otherwise.} \end{cases}$$
This corresponds to a graph on $n$ vertices having $\sim p \binom{n}{2}$ edges. It is a tightly connected cluster, there are no distinguished vertices and no underlying symmetries: each vertex plays essentially the same role. A random $k-$regular graph is simply a graph on $n$ vertices chosen uniformly at random from
the set of $k-$regular graphs on $n$ vertices. We note that in our setting of interest, $k = p \cdot n$ where $0 < p < 1$ is a fixed constant as $n \rightarrow \infty$, those two models are fairly similar with respect to many aspects: in the Erd\H{o}s-Renyi model, each vertex has $\sim p \cdot n \pm \mathcal{O}(\sqrt{n})$ neighbors. 
Assuming that the $p_{ij}$ come from an Erd\H{o}s-Renyi graph or a random $k-$regular graph has several advantages: once $n$ becomes large, there is an interesting regularization effect: for any two arbitrary subsets of vertices, as long as the number of vertices in each subset is not too small, we can estimate the number of edges that run between them fairly accurately (see also Fig. \ref{fig:reg}) and
$$ \# \left\{(a,b) \in E: a \in A \wedge b \in B \right\} \sim p \cdot |A| \cdot |B| \sim \frac{k}{n}\cdot |A| \cdot |B|.$$
This is a remarkable property because it implies that the associated energy should be primarily dominated by the distribution of $\left\{y_1, \dots, y_n\right\} \subset \mathbb{R}^2$: what is mainly relevant is \textit{how} the points are distributed in $\mathbb{R}^2$ rather than how the underlying graph behaves. As such, we expect that, for given fixed $\left\{y_1, \dots, y_n\right\} \in \mathbb{R}^2$, the t-SNE energy is essentially constant for a randomly chosen graph from the same model. 
Since it's virtually constant, it should be very well described by its expectation with the square root of the variance describing the typical deviation -- and both of these quantities, the expectation $\mathbb{E}$ and the variance $\mathbb{V}$ can be computed. \\

This underlying assumption, that for a fixed random graph model the t-SNE energy is being essentially given purely as a function of the embedding points $\left\{y_1, \dots, y_n\right\} \subset \mathbb{R}^2$ is naturally a key ingredient and quite similar to many other models in, say, statistical physics: the behavior of individual particles is assumed the even out and to give rise to an emerging mean field. The consequences of this assumption are quite impactful: the energy is then given merely as a function of a distribution of points in the plane and we end up trying to minimize $I(\mu)$, where $I$ is a notion of energy and $\mu$ ranges over all probability measures in $\mathbb{R}^2$. This is a much more classical mathematical problem and many more tools become available. In particular, the approach naturally extends to other such nonlinear dimensionality reduction methods such as SNE \cite{hint}, ForceAtlas2 \cite{jac}, LargeVis \cite{largevis} or UMAP \cite{umap}. We believe this to be a very promising avenue for further investigations.

\subsection{Emerging Functionals.} Suppose now that the affinities $p_{ij}$ are given by one of the two random models described above and suppose that $\left\{y_1, \dots, y_n\right\} \subset \mathbb{R}^2$ are given points in the plane. The t-SNE energy of this set of points (the functional we aim to minimize) will be shown to simplify (for the reason discussed in \S 3.2). We introduce some notation by introducing the probability measure $\mu$ in $\mathbb{R}^2$
$$\mu = \frac{1}{n} \sum_{k=1}^{n} \delta_{y_k} ,$$
We determine the approximation as 
$$ \mbox{t-SNE energy} = \mathbb{E} \pm  \sigma \sqrt{\mathbb{V}},$$
where both expectation $\mathbb{E}$ and variance $\mathbb{V}$ are computed with respect to the random Erd\H{o}s-Renyi model. These terms are fairly explicit: in particular, for the expectation we have (up to lower order errors)
$$ \mathbb{E} \sim 2p \binom{n}{2}  \left[\int_{\mathbb{R}^2} \int_{\mathbb{R}^2}  \|x-y\|^4 d\mu(x) d\mu(y) - \left( \int_{\mathbb{R}^2} \int_{\mathbb{R}^2} \|x-y\|^2 d\mu(x) d\mu(y)\right)^2 \right]$$
and we have a similar expression for the variance.
 Let now $\mu$ be an arbitrary probability measure on $\mathbb{R}^2$, then we will consider a renormalized energy given by
\begin{align*}
J_{\sigma, \delta}(\mu) &= \int_{\mathbb{R}^2} \int_{\mathbb{R}^2}   \|x-y\|^4 d\mu(x) d\mu(y) - \left( \int_{\mathbb{R}^2} \int_{\mathbb{R}^2} \|x-y\|^2 d\mu(x) d\mu(y)\right)^2\\
&+  \frac{\sigma}{ \sqrt{p}} \delta \left( \int_{\mathbb{R}^2} \int_{\mathbb{R}^2}  \|x - y\|^4 d \mu(x) d\mu(y)\right)^{1/2}.
\end{align*}
We expect that the behavior of the t-SNE energy for a random $k-$regular graph is approximately given by $J_{\sigma, \delta}$ with $\delta \sim n^{-1}$, $p \sim k/n$ and $\sigma$ a parameter at scale $\sim 1$.
The first two terms combined are always nonnegative: using the Cauchy-Schwarz inequality, we see that
$$  \int_{\mathbb{R}^2} \int_{ \mathbb{R}^2} \|x-y\|^4 d\mu(x) d\mu(y) -  \left( \int_{\mathbb{R}^2} \int_{\mathbb{R}^2} \|x-y\|^2 d\mu(x) d\mu(y)\right)^2 \geq 0.$$
The first two terms may thus be understood as a `Cauchy-Schwarz deficit' which then interacts with the two remaining terms. We show that minimizers can be characterized and have a peculiar shape.

\begin{thm} Among radial measures $\mu$, the functional $J_{\sigma, \delta}(\mu)$ has a unique minimizer (up to translation symmetry) given by the normalized arclength measure on a circle if $\sigma < 0$ or a Dirac measure centered in a point if $\sigma > 0.$ \end{thm}

Since $\delta = n^{-1}$, a more precise analysis of the scaling (done in the proof of the Theorem) would predict that the optimal t-SNE embedding of a random $k-$regular graph (assuming $k$ is proportional to $n$) to behave like a ring with radius $\sim k^{-1/4} n^{-1/4}$ (see Fig. \ref{fig:ring}). Numerical experiments support this conjecture and we see the final points arranged in a ring-like shape; however, it is more difficult to test the scaling since the decay is rather slow. Moreover, our derivation does employ a Taylor expansion: we would thus expect the scaling to be rather accurate once the ring is sufficiently small (say, diameter $\leq 0.001$). As seen in Fig. \ref{fig:ring} even for $n=40000$ and $k=4000$, the ring still has diameter $\sim 0.1$.

\subsection{Open Problems.} This motivates several interesting problems. 
\begin{enumerate} 
\item \textbf{Erd\H{o}s-Renyi model.} Can this type of analysis be carried out for the Erd\H{o}s-Renyi model? The difficulty lies in the quantity
$$\log{\left( - n + n^2 \int_{\mathbb{R}^2} \int_{\mathbb{R}^2} \frac{d\mu(x) d\mu(y)}{1+\|x-y\|^2}  \right)}$$
which, for fixed measure $\mu$ and $n \rightarrow \mathbb{R}$ scales like $\sim \log{n}$. We prove that, without this term, the measure localizes at
scale $\sim n^{-1/2}$ which is exactly the scale at which cancellation between the two terms occurs.
\item \textbf{Sparse Random Graphs.} We always work with $0 < p < 1$ or $k/n$ fixed, our Graphs are always dense. One could also wonder about the case where $p,k$ become smaller as $n$ gets larger: one would assume that randomness then starts playing a larger role. 
\item \textbf{Multiple clusters.} Can this analysis be extended to multiple clusters? The derivation of the functional itself is not difficult and can be carried out along the same lines. Numerical experiments suggest that the limit will not be a `ring'  but rather two more disk-like clusters. If the two clusters end up being close to one another, a Taylor expansion may be carried out, if they stay at a distance, then a new approach will be required. We emphasize that these problems are really problems regarding the structure of energy minimizers of certain functionals and, while possibly hard, quite classical.
\item \textbf{Other methods.} We would expect that the embedding of a random $k-$regular graph will asymptotically concentrate in a point for most of these methods. A similar analysis can then be conceivably carried out -- different functionals may come with different characteristic length scales which might be an interesting point for comparison. In particular, direct variations of t-SNE have been proposed (see e.g. \cite{kobak0}) for which the underlying analysis might be somewhat similar. As mentioned above, we believe this to be a promising line for further research.
\item \textbf{Expanders.} Other connections are conceivable. In particular, the regularization property of Erd\H{o}s-Renyi graphs that we use states that for any two subsets of vertices $A,B \in V$, the number of edges between $A$ and $B$ is proportional to $\sim p \cdot |A| \cdot |B|$. This property has in fact appeared in the context of expander graphs. More precisely, let $G = (V,E)$ be any $d-$regular random graph on $n$ vertices. Then, for any disjoint $A,B \subset V$, the Expander Mixing Lemma (see Alon \& Chung \cite{alon}) says that
$$ \left| \# \left\{ (a,b) \in E: a \in A \wedge b \in B \right\} - \frac{d}{n} |A| \cdot |B| \right| \leq \lambda \sqrt{|A| \cdot |B|},$$
where $\lambda$ is the second largest eigenvalue of the adjacency matrix $A_G$. This is exactly the type of regularity property used in our derivation -- it is thus conceivable that one might be able to derive rigorous bounds about the t-SNE embedding of an expander graph (though it might be less clear how one would generalize such an approach to the setting where there is more than one cluster).  
\end{enumerate}

\section{A Single Cluster: Proof of the Theorem}
\S 4 contains our theoretical arguments: \S 4.1 formally introduces t-SNE, \S 4.2 computes the mean field (along the lines discussed in \S 3.2), and \S 4.3 shows that this is sufficient to deduce that, as $n \rightarrow \infty$, the optimal t-SNE embedding of an Erd\H{o}s-Renyi random graph will shrink in diameter. This shrinkage will allow us to apply a Taylor approximation in \S 4.4 which will be shown to have a scaling symmetry in \S 4.5. Finally, we prove the Theorem in \S 4.6.

\subsection{The t-SNE functional.} We quickly recall the t-SNE algorithm. Given a set of points $\mathcal{X} = \{x_1, x_2, ...,  x_n\} \subset \mathbb{R}^d$, we define the affinity $p_{ij}$ between any pair as

$$
\label{p_ij} p_{ij}  = \frac{ p_{i|j} + p_{j|i} }{2n},  \qquad \mbox{where} \qquad p_{i|j} = \frac{\exp{(-\|  x_i -  x_j \|^2 /2 \sigma_i^2 )}}{\sum_{k\neq i} \exp{( - \|  x_i -  x_k \|^2/ 2 \sigma_i^2 }) }.
$$

The parameters $\sigma_i$ are usually set based on the local scale of the neighborhood of $x_i$. This expression for $p_{ij}$ will not be relevant to our analysis. Assume now that $\mathcal{Y} = \left\{y_1, \dots, y_n\right\} \subset \mathbb{R}^s$ is our embedding. We describe a notion of energy that aims to quantify how `similar' the points $x_i$ and the points $y_i$ are. For this, we define the analogue of the $p_{ij}$: the quantity $q_{ij}$ will denote a notion of similarity between points $ y_i$ and $ y_j$ via

$$
q_{ij} = \frac{(1 + \| y_i -  y_j\|^2)^{-1}}{\sum_{k\neq \ell} (1 +\|  y_k -  y_\ell\|^2 )^{-1}}.
$$

The energy is then defined as

$$
E = \sum_{i,j=1 \atop i\neq j}^n p_{ij} \log \frac{p_{ij}}{q_{ij}}.
$$

t-SNE then uses gradient descent to find an output $\mathcal{Y}$ which minimizes $E$. The remaining question is how to initialize $\mathcal{Y}$ for this step. Early implementations of t-SNE chose $\mathcal{Y}$ uniformly at random. However, there is evidence that initializing with the results of another dimensionality reduction method can better preserve global structure in the final embedding (see Kobak \& Linderman \cite{kobak3}). In light of our arguments above (especially \S 3.2), it is clear that initialization will not be important in our argument. 
Finally, we observe that $E$ is solely a function of the $q_{ij}$ during the gradient descent optimization, since the $p_{ij}$ are constants determined by the input data. This naturally raises the question of whether other functions for $q_{ij}$ besides the one provided could produce interesting results. Empirical results have shown that this is indeed the case, as the decay of the function seems to correspond to the resolution with which a cluster's substructure is displayed (see \cite{kobak0}). However, many other choices of functionals are possible, and many result in methods that work quite well. We refer to B\"ohm, Behrens \& Kobak \cite{bohm} for a unifying overview. 
In our analysis of t-SNE, we will fix the embedding dimension $s = 2$, and we will set $p_{ij}$ as the adjacency matrix of an Erd\H{o}s-Renyi or random $k$-regular graph. The original $p_{ij}$ form a probability distribution while our values are in $\{0, 1\}$. This is not an issue because, as is easily seen from the structure of the energy $E$, the local minima of $E$ are invariant under rescaling of $p_{ij}$.

\subsection{Mean Fields.} \label{sec:erdmean} This section discusses the first approximations. 
We would like to understand the behavior of the following functional for large $n$:
$$ 
\sum_{i,j=1 \atop i\neq j}^n p_{ij} \log{\frac{p_{ij}}{q_{ij}}} =  \sum_{i,j=1 \atop i\neq j}^n p_{ij} \log{p_{ij}} + \sum_{i,j=1 \atop i\neq j}^n p_{ij} \log{\frac{1}{q_{ij}}} \rightarrow \min.
$$
The first of these two sums only depends on $p_{ij}$, which are externally given and independent of the actual embedding of points in $\mathbb{R}^2$. We can thus safely ignore the first sum containing only $p_{ij}$ quantities. It remains to understand the second sum.
Plugging in the definition of $q_{ij}$ yields: 
$$ 
\boxed{ \sum_{i,j=1 \atop i \neq j}^n p_{ij} \log \left( \sum_{k, \ell =1 \atop k \neq \ell}^n  \frac{1}{1+\|y_{\ell} - y_k\|^2} \right) + \sum_{i,j=1 \atop i \neq j}^n p_{ij} \log{(1 + \|y_i - y_j\|^2)} \rightarrow \min.}
$$
This is the problem that we will analyze for the remainder of the paper.  We observe that the functional is comprised of two terms. We will compute the expectation and variance for both.
\begin{enumerate}
\item For \textbf{k-regular graphs}, the first term simplifies because the number of edges is constant ($|E| = kn/2$). The inner sum does not depend on $i,j$ and
$$
\sum_{i,j=1 \atop i \neq j}^n p_{ij} \log \left( \sum_{k, \ell =1 \atop k \neq \ell}^n  \frac{1}{1+\|y_{\ell} - y_k\|^2} \right) = kn \cdot \log \left( \sum_{k, \ell =1 \atop k \neq \ell}^n  \frac{1}{1+\|y_{\ell} - y_k\|^2} \right).
$$
\item In the \textbf{Erd\H{o}s-Renyi model}, the first term has roughly the same expectation because
$$
\mathbb{E} \sum_{i,j=1 \atop i \neq j}^n p_{ij}  = 2p \binom{n}{2}
$$
but a nontrivial variance (computed below).
\end{enumerate}
As above, we simplify notation by introducing the probability measure
$$ 
\mu = \frac{1}{n} \sum_{i=1}^{n}{ \delta_{y_i}}.
$$
\subsubsection{The first term.}
We can write the expectation of the first term for the Erd\H{o}s-Renyi random model as
$$
\mathbb{E} \sum_{i,j=1 \atop i \neq j}^n p_{ij} \log \left( \sum_{k, \ell =1 \atop k \neq \ell}^n  \frac{1}{1+\|y_{\ell} - y_k\|^2} \right) = 2p \binom{n}{2} \log{\left( - n + n^2 \int_{\mathbb{R}^2} \int_{\mathbb{R}^2} \frac{d\mu(x) d\mu(y)}{1+\|x-y\|^2}  \right)}.
$$
Replacing $2p \binom{n}{2}$ with $kn$ leads to the expectation for $k$-regular graphs (which is actually a constant). 
As evidenced by the numerical examples above and the arguments in \S 4.3, in the asymptotic regime the measure $\mu$ will concentrate around a point. This means that we expect the integral to be size $\sim 1$ which turns the $-n$ factor inside the logarithm into a lower order term -- as it turns out, it will be structurally similar to other lower order terms and can be absorbed by them. 
For simplicity of exposition, we first ignore the lower order term and use the algebraically more convenient approximation
$$
\mathbb{E}_2  = 2p \binom{n}{2} \log{\left( n^2 \int_{\mathbb{R}^2} \int_{\mathbb{R}^2} \frac{d\mu(x) d\mu(y)}{1+\|x-y\|^2}  \right)}
$$
which can also be written as
$$
\mathbb{E}_2 =2p \binom{n}{2} \log{(n^2)} + 2p \binom{n}{2}  \log{\left(  \int_{\mathbb{R}^2} \int_{\mathbb{R}^2} \frac{d\mu(x) d\mu(y)}{1+\|x-y\|^2}  \right)},
$$
where only the second term depends on $\mu$. In Section \ref{subsec:pert}, we show how to compare $\mathbb{E}$ and $\mathbb{E}_2$.
In the $k$-regular case, the variance is 0 because this term is constant. In the Erd\H{o}s-Renyi case, this is slightly different but it is relatively easy to compute
$$
\mbox{the variance} \qquad \mathbb{V} \sum_{i,j=1 \atop i \neq j}^n p_{ij} \log \left( \sum_{k, \ell =1 \atop k \neq \ell}^n  \frac{1}{1+\|y_{\ell} - y_k\|^2} \right).
$$
Recall that the variance of a sum of independent random variables is given by the sum of the variances and that for a random variable $X \sim \text{Bern}(p)$, we have $\mathbb{V} X = p(1-p)$. Therefore,
 $$ \mathbb{V}  = 2p (1-p) \binom{n}{2} \left[ \log{\left( - n + n^2 \int_{\mathbb{R}^2} \int_{\mathbb{R}^2} \frac{d\mu(x) d\mu(y)}{1+\|x-y\|^2}  \right)}\right]^2.$$
\subsubsection{The Second Term.}
It remains to analyze the second term.
$$ 
\sum_{i,j=1 \atop i \neq j}^n p_{ij} \log{(1 + \|y_i - y_j\|^2)}
$$
This term is more involved since it couples $p_{ij}$ with the location of $y_i$ and $y_j$ in $\mathbb{R}^2$. However, we are able to treat the random $k-$regular model and the Erd\H{o}s-Renyi model simultaneously. 
Taking two arbitrary subsets of vertices, the number of edges between them cannot deviate substantially from the expectation. Let us now assume, more precisely, that $B_1, B_2 \subset \mathbb{R}^2$ are two small disjoint boxes in $\mathbb{R}^2$. The number of points in $B_1$ is given by $n \cdot \mu(B_1)$, the number of points in $B_2$ is given by $n \cdot \mu(B_2)$. 
Since the underlying graph is Erd\H{o}s-Renyi, we have that the expected number of edges with one vertex in $B_1$ and the other in $B_2$ is
$$
\mathbb{E} ~\sum_{v \in B_1} \sum_{w \in B_2} 1_{(v,w) \in G} =  p n^2 \mu(B_1) \mu(B_2).
$$
For $k$-regular graphs, we use that 
$$\mathbb{E}[1_{(v,w) \in G}] = \frac{kn/2}{\binom{n}{2}} = \frac{k}{n-1}$$
 instead. Since $k/(n-1) \sim p$, this is essentially the same as with Erd\H{o}s-Renyi graphs.
In the next step, we will compute the variance. The variance of a sum of independent random variables is the sum of the variances of each individual random variable. Since the variance of Bernoulli random variables with likelihood $p$ is given by $p(1-p)$, we obtain
$$ 
\mathbb{V} ~\sum_{v \in B_1} \sum_{w \in B_2} 1_{(v,w) \in E} =   p(1-p)n^2 \mu(B_1) \mu(B_2).
$$
From this we get that taking the expectation with respect to all Erd\H{o}s-Renyi random graphs for fixed $\left\{y_1, \dots, y_n\right\} \subset \mathbb{R}^2$ leads to,
$$
\mathbb{E} \sum_{i,j=1 \atop i \neq j}^n p_{ij} \log{(1 + \|y_i - y_j\|^2)} =   2p \binom{n}{2} \int_{\mathbb{R}^2} \int_{\mathbb{R}^2} \log{(1 + \|x - y\|^2)} d \mu(x) d\mu(y).
$$
We recall that when dealing with the expectation $\mathbb{E}$, switching to the integral required taking out self-interactions (resulting in $\mathbb{E}_2$ and an analysis to be done in \S \ref{subsec:pert}). Self interactions do not contribute here since $\log(1 + \|y_i - y_i\|^2) = 0$. By the same approach, we can compute the variance with respect to Erd\H{o}s-Renyi random graphs and arrive at
$$ \mathbb{V} \sum_{i,j=1 \atop i\neq j}^n p_{ij} \log{(1 + \|y_i - y_j\|^2)} =   p(1-p) n^2\int_{\mathbb{R}^2} \int_{\mathbb{R}^2} \log{(1 + \|x - y\|^2)}^2 d \mu(x) d\mu(y).$$
It remains to compute the variance of this term with respect to the model of $k-$regular random graphs. Naturally, we expect this to be very close to the variance for the Erd\H{o}s-Renyi model. The main difference is that the $p_{ij}$ are no longer independent random variables but exhibit a slight negative correlation. We make use of the following Lemma.

\begin{lemma}
Let $\left\{ i, j, k, l\right\}$ denote four different vertices. Then, with respect to all random $k-$regular graphs on all $n$ vertices, there exist two positive quantities $c_{k,n} \sim 1 \sim c_{2,k,n}$ (comparable to a universal constant) such that
$$ \mathbb{E} ~p_{i,j}  p_{i,k}  \sim \frac{k^2}{n^2} -  c_{k,n} \frac{k^2}{n^3}$$
and
$$ \mathbb{E} ~p_{i,j}   p_{k,l} = \frac{k^2}{n^2} -  c_{2,k,n} \frac{k^2}{n^3}.$$
\end{lemma}
\begin{proof} We start with the first case. This is simply asking about the likelihood that two fixed edges $(i,j)$ and $(i,k)$ emanating from the same vertex both end up in a random $k-$regular graph. Since everything is invariant under relabeling and the product of two indicator functions is only 1 when both are 1
$$ \mathbb{E} ~p_{i,j}  p_{i,k} = \frac{k}{n-1} \frac{k-1}{n-2}.$$
An alternative argument would proceed as follows: the ways of choosing $k$ elements out of $n-1$ with 2 elements being fixed is given by
$$ \frac{\binom{n-3}{k-2}}{\binom{n-1}{k}} = \frac{\frac{(n-3)!}{ (n-k)! (k-2)!}}{ \frac{(n-1)!}{k! (n-1 -k)!}} = \frac{k (k-1)}{(n-1) (n-2)}.$$
Let us now consider the likelihood of two disjoint edges $p_{i,j} p_{k,l}$ being contained in the random graph. We use, since these are indicator functions,
$$ \mathbb{E} ~p_{i,j}   p_{k,l}  = \mathbb{P}( p_{i,j} p_{k,l} = 1) = \mathbb{P}\left( p_{k,l} = 1 \big| p_{i,j}=1\right) \cdot \mathbb{P}(p_{i,j} = 1).$$
We have $ \mathbb{P}(p_{i,j} = 1) = k/(n-1)$, it remains to understand the conditional expectation. Suppose that $p_{i,j} = 1$. 

\begin{center}
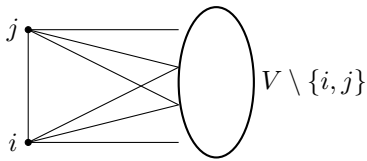
\begin{figure}[h!]
\begin{tikzpicture}
\filldraw (0,0) circle (0.04cm);
\filldraw (0,1.5) circle (0.04cm);
\foreach \x in {0,1,2}
      {
        \draw (0,0) -- (2, 0.5*\x);
                \draw (0,1.5) -- (2, 1.5-0.5*\x);
      }
\node at (-0.2, 0) {$i$};      
\node at (-0.2, 1.5) {$j$};      
\draw [thick] (2.5,0.8) ellipse (0.5cm and 1cm);
\draw (-0, 0) -- (-0, 1.5);
\node at (3.8, 0.8) {$V \setminus \left\{i,j\right\}$};
\end{tikzpicture}
\caption{Sketch of the Argument.}
\end{figure}
\end{center}
\vspace{-10pt}
The symmetry of the $k-$random regular graphs allows us to always relabel vertices in the complement. The likelihood of $p_{k,l} = 1$ subject to $p_{i,j} =1$ is thus simply determined by the total number of edges between vertices in $V \setminus \left\{i, j\right\}$. There are $n \cdot k$ edges in total. Of those $k-1$ connect $i$ and $V \setminus \left\{i, j\right\}$ and $k-1$ connect $j$ and $V \setminus \left\{i, j\right\}$. Thus the subgraph induced by the vertices $V \setminus \left\{i,j\right\}$ has $nk/2 - 2k + 1$ edges. The likelihood of any random pair of vertices being connected is thus
$$ \mathbb{P}\left(p_{k,l} \big| p_{i,j} = 1\right) = \frac{nk/2 - 2k + 1}{\binom{n-2}{2}}.$$
Thus
$$ \mathbb{E} p_{i,j} p_{k,\ell} = \frac{k}{n-1} \frac{nk/2 - 2k + 1}{\binom{n-2}{2}} \sim \frac{k^2}{n^2} - 4 \frac{k^2}{n^3}.$$
\end{proof}

These expectations $ \mathbb{E} ~p_{i,j}  p_{i,k}$ and $ \mathbb{E} ~p_{i,j}   p_{k,l} $ are very close to what one would expect for independently chosen edges. This suggests that our computation of the variance assuming the Erd\H{o}s-Renyi model should be close to the truth and indeed it is.

\begin{lemma} Assuming this type of correlation structure, for arbitrary $x_{i,j} \in \mathbb{R}$,
$$ \mathbb{V} \sum_{i,j} p_{i,j} x_{i,j}  \geq  \frac{k}{n} \left(1 - \frac{k}{n}\right) \sum_{i,j} x_{i,j}^2 - c_{k,n} \frac{k^2}{n^3} \left( \sum_{i,j}  x_{i,j}\right)^2.$$
\end{lemma}
\begin{proof} We have $\mathbb{V} X = \mathbb{E} X^2 - \left( \mathbb{E} X \right)^2$. Assuming the $p_{i,j}$ are independent variables that are 1 with likelihood $k/n$ and $0$ with likelihood $1-k/n$, we see
$$ \mathbb{V} \sum_{i,j} p_{i,j} x_{i,j} = \sum_{i,j} x_{i,j}^2 \mathbb{V} p_{i,j} = \frac{k}{n} \left(1 - \frac{k}{n}\right) \sum_{i,j} x_{i,j}^2.$$
Let us now assume that they are not independent but almost independent in the sense above. Then
\begin{align*}
 \mathbb{E} \left( \sum_{i,j} p_{i,j} x_{i,j} \right)^2 &= \mathbb{E} \sum_{i_1, i_2, j_1, j_2}  p_{i_1,j_2} x_{i_1,j_2}  p_{i_2,j_2} x_{i_2,j_2} \\
 &= \sum_{i_1, i_2, j_1, j_2}  x_{i_1,j_2} x_{i_2,j_2} \mathbb{E} p_{i_1,j_2} p_{i_2,j_2}.
 \end{align*}
 This leads to exactly the same quantity as above except for an additional error term of size
 $$ \sim - c_{k,n} \frac{k^2}{n^3} \left( \sum_{i,j}  x_{i,j}\right)^2.$$
\end{proof}
A simple computation shows that we expect the constant to scale roughly like $c_{k,n} \sim 4$.
Combining all these ingredients, we expect the variance of the second term with respect to random $k-$regular graphs to be given by
\begin{align*}
 \mathbb{V} \sum_{i,j=1 \atop i\neq j}^n p_{ij} \log{(1 + \|y_i - y_j\|^2)} &\sim   p(1-p) n^2\int_{\mathbb{R}^2 \times \mathbb{R}^2} \log{(1 + \|x - y\|^2)}^2 d \mu(x) d\mu(y)\\
 &- c_{k,n} p(1-p) n \left(\int_{\mathbb{R}^2 \times \mathbb{R}^2}  \log{(1 + \|x - y\|^2)}d \mu(x) d\mu(y) \right)^2.
\end{align*}
The second term in this approximation is indeed a lower order perturbation as $n$ becomes large. The first term is exactly what is predicted by the Erd\H{o}s-Renyi random model.

% correct coefficients in shrinkage section if prev corrections are ok
\subsection{Shrinkage.} \label{sec:shrinkage} Taking the leading terms derived in the prior section, we have (up to lower order terms)
\begin{align*}
 \mathbb{E} ~\mbox{t-SNE loss} &= 2p \binom{n}{2} \log{(n^2)}  + 2p  \binom{n}{2}  \log{\left( \int_{\mathbb{R}^2} \int_{\mathbb{R}^2} \frac{ d\mu(x) d\mu(y)}{1+\|x-y\|^2}   \right)} \\
 &+ 2p \binom{n}{2}\int_{\mathbb{R}^2} \int_{\mathbb{R}^2} \log{(1 + \|x - y\|^2)} d \mu(x) d\mu(y).
 \end{align*}
The constant is the same in front of all three terms. Moreover, the first term is constant, depending only on $n$ and $p$ and thus irrelevant for the study of minimizing configurations.
When studying the minimizer, it thus suffices to consider the rescaled functional
$$ I(\mu) =  \log{\left( \int_{\mathbb{R}^2} \int_{\mathbb{R}^2} \frac{ d\mu(x) d\mu(y)}{1+\|x-y\|^2}   \right)} + \int_{\mathbb{R}^2} \int_{\mathbb{R}^2} \log{(1 + \|x - y\|^2)} d \mu(x) d\mu(y).$$
The logarithm is concave and thus we have by Jensen's inequality that
 $$  \log{\left( \int_{\mathbb{R}^2} \int_{\mathbb{R}^2} \frac{ d\mu(x) d\mu(y)}{1+\|x-y\|^2}   \right)} \geq   
 \int_{\mathbb{R}^2} \int_{\mathbb{R}^2} \log \left( \frac{ 1}{1+\|x-y\|^2}   \right) d\mu(x) d\mu(y)$$
from which we deduce 
$ I(\mu) \geq 0$
with equality if and only if all the mass is concentrated in one point, i.e. $\mu = \delta_{x_0}$ for some $x_0 \in \mathbb{R}^2$. This already illustrates part of the dynamics that plays out: the expected term (not considering lower order perturbations) has a Jensen-type structure and forces the measure to contract -- this is counter-balanced by quantities coming from the variance and the interactions between these two terms leads to final estimates about the scale of the measure.
We quickly establish a quantitative result that essentially states that if $\mu$ is spread out over a small scale (in a certain sense), then $I(\mu)$ has to be strictly bigger than 0. (The proof will also show that if $\mu$ is spread out over a large area, then much stronger estimates holds, we will not concentrate on that part). We define a length-scale $r(\mu)$ of any probability measure $\mu$ on $\mathbb{R}^2$ via
$$ r(\mu) = \inf \left\{ \mbox{sidelength}(Q): Q \subset \mathbb{R}^2, Q~\mbox{square}, ~\mu(Q) \geq \frac{1}{200} \right\}.$$
If $r(\mu) = 0$ (something that happens if $99.5\%$ of all mass is concentrated in a point, for example, which is not something that we would expect for the minimizers of t-SNE energy), then the following Lemma does not say anything interesting but one can simply replace $1/200$ by $0 \leq \delta \ll 1$ and rerun the argument. In the interest of clarity of exposition, we fix $\delta = 1/200$ in the definition of $r(\mu)$.
\begin{lemma}  Let $\mu$ be a probability measure on $\mathbb{R}^2$. Then, for some universal $c>0$,
$$ I(\mu) \geq  c\frac{r(\mu)^4}{(1+ r(\mu)^2)^4}.$$
\end{lemma}
We note that we are only interested in the case when the measure is already quite concentrated, i.e. $r(\mu)$ small. In that case, the lower bound is really $\gtrsim r(\mu)^4$ and, before proving the statement, we quickly show that it has the sharp scaling. Let $\mu$ be the sum of two Dirac measures, each having mass $1/2$ and being at distance $r$ from each other. Then a quick computation shows that
$$ I(\mu) = \log\left(\frac{1}{2} + \frac{1}{2} \frac{1}{1+r^2}\right) + \frac{1}{2} \log{(1+r^2)} \sim \frac{r^4}{8} + \mbox{l.o.t.} \qquad \mbox{as}~r \rightarrow 0.$$

\begin{proof} We start by using a refined version of Jensen's inequality (see \cite{liao}). If $\nu$ is a probability measure supported on $[0,1]$ and $Z$ is a random variable following that distribution, then
$$ \log \left( \mathbb{E} Z \right) - \int_{0}^{1} \log{(x)} d\nu(x) \geq \frac{1}{2} \mathbb{V} Z.$$
As will come as no surprise, this is a consequence of the strict uniform bound on the second derivative of $\log$ in the interval $(0,1)$: stronger results would be available if $\mathbb{E}Z$ is close to 0 (see \cite{liao}) but we are interested in the case where $\mathbb{E}Z$ is fairly close to 1. Let us now return to our original setting: given a measure $\mu$, we will consider the random variable
$$ Z = \frac{1}{1+\|X-Y\|^2}, ~\mbox{where} \quad X \sim \mu \sim Y$$
are two independent realizations of $\mu$. Then
$$ I(\mu) =  \log \left( \mathbb{E} Z \right) -  \mathbb{E} \log{Z}.$$
We can now introduce the induced measure $\nu$ describing the distribution of $Z$ in the unit interval via
$$ \nu(A) = \int_{\mathbb{R}^2} \int_{\mathbb{R}^2} 1_{\frac{1}{1+\|x-y\|^2} \in A} ~ d\mu(x) d\mu(y).$$
Appealing to the strenghtened Jensen inequality, we deduce
$$ 2 \cdot I(\mu) \geq \mathbb{V} Z.$$

\begin{center}
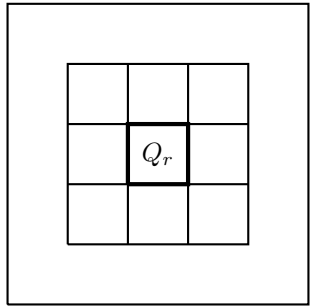
\begin{figure}[h!]
\begin{tikzpicture}[scale=0.8]
\draw [ultra thick] (0,0) -- (1,0) -- (1,1) -- (0,1) -- (0,0);
\draw [thick] (-1, -1) -- (2,-1) -- (2,2) -- (-1, 2) -- (-1, -1);
\draw [thick] (-1, 0) -- (2, 0);
\draw [thick] (-1, 1) -- (2, 1);
\draw [thick] (0, -1) -- (0, 2);
\draw [thick] (1, -1) -- (1, 2);
\draw [thick] (-2, -2) -- (3, -2) -- (3,3) -- (-2,3) -- (-2,-2);
\node at (0.5, 0.5) {$Q_r$};
\end{tikzpicture}
\caption{The densest square $Q_r$.}
\end{figure}
\end{center}
It remains to understand how the variance of $Z$ depends on the distribution properties of $\mu$: more precisely, what we want to show is that if $\mu$ is not concentrated around a single point and $X \sim \mu \sim Y$, then the random variable
$$ Z = \frac{1}{1 + \|X-Y\|^2} \qquad \mbox{cannot be concentrated around a single value.}$$ 
This, in turn, is equivalent to showing that $\|X - Y\|$ cannot be concentrated around a single value. This is where our notion of $r(\mu)$ comes into play. We first assume that the smallest square containing $1/200$ of the total mass has positive side-length $r(\mu)$. Let us call this smallest square $Q_r$.
We note that $\mu(Q_r) \leq 1/50$: if it were larger, then we could subdivide $Q_r$ into four smaller squares at least one of which will have measure $1/200$ which is a contradiction. We take the $24 = 5^2 - 1$ squares of equal sidelength surrounding $Q_r$: since all of these squares have total measure less than $1/50$, this means that at least half the measure is outside this $5 \times 5$ box and thus at least distance $2r(\mu)$ from any point in $Q_r$. This shows that 
$$ \mathbb{P} \left( \|X - Y \| \leq \sqrt{2}  r(\mu)\right) \geq \frac{1}{200} \cdot \frac{1}{200} \quad \mbox{and} \quad  \mathbb{P} \left( \|X - Y \| \geq 2  r(\mu) \right) \geq \frac{1}{200} \cdot \frac{1}{2}.$$
This proves that
$$ \mathbb{P} \left( \frac{1}{1+ \|X - Y \|^2} \geq \frac{1}{1 + 2 r(\mu)^2} \right) \geq \frac{1}{200} \cdot \frac{1}{200}$$
as well as
$$ \mathbb{P} \left( \frac{1}{1+\|X - Y \|^2} \leq \frac{1}{1+4r(\mu)^2}  \right) \geq \frac{1}{200} \cdot \frac{1}{2}.$$
At this point we use the following simple Lemma: if $a<b$ and $\mathbb{P}(X \leq a) \geq c_1$ and $\mathbb{P}(X \geq b) \geq c_2$, then $\mathbb{V} X \gtrsim (b-a)^2,$ where the implicit constant depends on $c_1$ and $c_2$.
This shows that, for some implicit universal constant, 
$$ \mathbb{V} Z \gtrsim \frac{r(\mu)^4}{(1+ r(\mu)^2)^4}.$$
\end{proof}

\textit{A Curious Problem.} We quickly note the following curious problem that arises naturally in the context of the Lemma. Let $\Omega \subset \mathbb{R}^n$ be given and let $X_1,X_2$ be two independent random variables sampled uniformly from $\Omega$, i.e. for each $A \subset \mathbb{R}^n$
$$ \mathbb{P}\left(X_i \in A\right) = \frac{|A \cap \Omega|}{|\Omega|}.$$
It is an interesting question to understand the behavior of $\mathbb{E} \| X - Y\|$, especially the question of how small this expectation can be. It is known \cite{blaschke, pfiefer} that $$\mathbb{E} \|X - Y\| \geq c_d |\Omega|^{1/d}$$ and that the extremal case is given by the ball (in light of the Riesz Rearrangement Inequality, this is perhaps not too surprising). We also refer to the more recent results \cite{thal, burg}. 
One could naturally ask whether there is an analogous inequality for the variance, i.e. $$\mathbb{V} \|X-Y\| \geq c_d \cdot |\Omega|^{2/d}$$ and whether it is possible to identify the optimal shape. Is it again a ball?

\subsection{Taylor Expansions.} What we have seen in the preceding section is that the main parts of the t-SNE energy functional (assuming a random underlying graph) will ultimately lead to a shrinking of the cluster down to a small area in space. This allows us to further simplify all the functionals by replacing them with their Taylor expansions. We have three terms (two expectations and one variance). 

We recall that one of the expectations, $\mathbb{E}_2$, is a lower order perturbation of the true expectation $\mathbb{E}$. We will first perform a Taylor expansion of the algebraically simpler quantity $\mathbb{E}_2$ before showing, in \S \ref{subsec:pert} that the difference between $\mathbb{E}$ and $\mathbb{E}_2$ can be absorbed in already existing lower order terms.

\subsubsection{Expectations.} We start by
analyzing the expectations, i.e.
$$  \log{\left( \int_{\mathbb{R}^2} \int_{\mathbb{R}^2} \frac{ d\mu(x) d\mu(y)}{1+\|x-y\|^2}   \right)}  \quad \mbox{and} \quad \int_{\mathbb{R}^2} \int_{\mathbb{R}^2} \log{(1 + \|x - y\|^2)} d \mu(x) d\mu(y)$$
under the assumption that all the mass is contained in a ball of radius $r$ centered around a fixed point (due to the translation invariance of these functionals, it does not matter where that point is).
For the first term, since $\mu$ is a probability measure, we have
$$ \int_{\mathbb{R}^2} \int_{\mathbb{R}^2} \frac{ d\mu(x) d\mu(y)}{1+\|x-y\|^2}   = 1 - \int_{\mathbb{R}^2} \int_{\mathbb{R}^2} \frac{ \|x-y\|^2}{1+\|x-y\|^2}    d\mu(x) d\mu(y)$$
which we expect to be $\mathcal{O}(r^2)$ close to 1. Using the Taylor expansion of the logarithm around 1
$$ \log{(1+x)} = x - \frac{x^2}{2} + \frac{x^3}{3} - \frac{x^4}{4} + \dots,$$
we can expand the integral as
\begin{align*}
  \log{\left( \int_{\mathbb{R}^2} \int_{\mathbb{R}^2} \frac{d\mu(x) d\mu(y)}{1+\|x-y\|^2}  \right)}  &=  - \int_{\mathbb{R}^2} \int_{\mathbb{R}^2} \frac{\|x-y\|^2}{1+\|x-y\|^2} d\mu(x) d\mu(y) \\
  & -  \frac{1}{2} \left( \int_{\mathbb{R}^2} \int_{\mathbb{R}^2} \frac{\|x-y\|^2}{1+\|x-y\|^2} d\mu(x) d\mu(y)\right)^2+ \mathcal{O}(r^6).
\end{align*}
We simplify the first integral using
$$  \frac{\|x-y\|^2}{1+\|x-y\|^2} =  \|x-y\|^2 - \|x-y\|^4 + \mathcal{O}(\|x-y\|^6).$$
The second integral is already $\mathcal{O}(r^4)$ and we can thus perform the same simplification. This leads to
\begin{align*}
  \log{\left( \int_{\mathbb{R}^2} \int_{\mathbb{R}^2} \frac{d\mu(x) d\mu(y)}{1+\|x-y\|^2}  \right)}  &=  - \int_{\mathbb{R}^2} \int_{\mathbb{R}^2} \|x-y\|^2 d\mu(x) d\mu(y) \\
  &+ \int_{\mathbb{R}^2} \int_{\mathbb{R}^2} \|x-y\|^4 d\mu(x) d\mu(y) \\
  & -  \frac{1}{2} \left( \int_{\mathbb{R}^2} \int_{\mathbb{R}^2} \|x-y\|^2 d\mu(x) d\mu(y)\right)^2+ \mathcal{O}(r^6).
\end{align*}
For the second expectation, we again use the Taylor expansion of the logarithm leading to
\begin{align*}
\int_{\mathbb{R}^2} \int_{\mathbb{R}^2}  \log{(1+\|x-y\|^2)} d\mu(x) d\mu(y) &= \int_{\mathbb{R}^2} \int_{\mathbb{R}^2}  \|x-y\|^2 d\mu(x) d\mu(y)\\
&- \frac{1}{2} \int_{\mathbb{R}^2} \int_{\mathbb{R}^2}  \|x-y\|^4 d\mu(x) d\mu(y) + \mathcal{O}(r^6).
\end{align*}
Summing both of these terms up, we obtain as the Taylor expansion of the expected t-SNE energy (when averaging over all random graphs)
\begin{align*}
\mathbb{E}_2 &= pn(n-1) \log{(n^2)} + pn \int_{\mathbb{R}^2} \int_{\mathbb{R}^2} \|x-y\|^2 d\mu(x) d\mu(y) \\
&+ \dfrac{pn(n-2)}{2} \int_{\mathbb{R}^2}  \int_{\mathbb{R}^2}   \|x-y\|^4 d\mu(x) d\mu(y) \\
&-  \dfrac{pn(n-1)}{2} \left( \int_{\mathbb{R}^2} \int_{\mathbb{R}^2} \|x-y\|^2 d\mu(x) d\mu(y)\right)^2 + \mathcal{O}(r^6).
\end{align*}
This concludes our expansion of the terms controlling the expectations.

\subsubsection{Variance.} The variance depends on the random graph model. As derived above, in the Erd\H{o}s-Renyi case, we have that the variance of the first term satisfies
 $$ \mathbb{V} = 2p (1-p) \binom{n}{2} \left[ \log{\left(- n + n^2 \int_{\mathbb{R}^2} \int_{\mathbb{R}^2} \frac{d\mu(x) d\mu(y)}{1+\|x-y\|^2}  \right)}\right]^2.$$
 This quantity is a priori, for fixed $\mu$ and $n$ becoming large, an object at scale $\sim_p n^2 (\log{n})^2$ which is larger than the second term. One would expect that this actually tells us something about the size of the integral: presumably it will
 actually be much smaller so that the variance is not quite as large. In fact, one would perhaps believe that the integral is of such a size that the logarithm becomes small, this would suggest
 $$  \int_{\mathbb{R}^2} \int_{\mathbb{R}^2} \frac{d\mu(x) d\mu(y)}{1+\|x-y\|^2}   \sim \frac{1}{n}$$
 which would indicate that $\mu$ is distributed over a scale of size $\sim n^{-1/2}$ which is the scaling we get in the case of $k-$regular random graphs. It is clear that this case presents with some interesting dynamics; it would be desirable to have a better understanding.
However, switching to the case of random $k-$regular graphs, we see that there is only one variance term (the variance of the second term in the energy) and that this term is given by
$$ \mathbb{V} =  p(1-p) n^2 \int_{\mathbb{R}^2} \int_{\mathbb{R}^2} \log{(1 + \|x - y\|^2)}^2 d \mu(x) d\mu(y).$$
A Taylor expansion up to $\mathcal{O}(r^6)$ shows that 
$$ \mathbb{V} =  p(1-p) n^2 \int_{\mathbb{R}^2} \int_{\mathbb{R}^2}  \|x - y\|^4 d \mu(x) d\mu(y) + \mathcal{O}(r^6).$$

\subsubsection{Adding a slight perturbation.} \label{subsec:pert} We will now compare the true expectation of the first term, it being
$$ \mathbb{E} =  2p \binom{n}{2} \log{\left( - n + n^2 \int_{\mathbb{R}^2} \int_{\mathbb{R}^2} \frac{d\mu(x) d\mu(y)}{1+\|x-y\|^2}  \right)}$$
to the algebraically more convenient approximation
$$ \mathbb{E}_2 = 2p \binom{n}{2} \log{\left( n^2 \int_{\mathbb{R}^2} \int_{\mathbb{R}^2} \frac{d\mu(x) d\mu(y)}{1+\|x-y\|^2}  \right)}$$
that we have used up to now. The mean value theorem implies that for $ 0 < |y| \ll x$, we have
$$\log{(x+y)} \sim \log{x} + \frac{y}{x} + \mathcal{O}\left(\frac{y^2}{x^2}\right)$$
and therefore
\begin{align*}
\log{\left( - n + n^2 \int_{\mathbb{R}^2} \int_{\mathbb{R}^2} \frac{d\mu(x) d\mu(y)}{1+\|x-y\|^2}  \right)} &= \log{\left(n^2 \int_{\mathbb{R}^2} \int_{\mathbb{R}^2} \frac{d\mu(x) d\mu(y)}{1+\|x-y\|^2}  \right)} \\
&- \frac{1}{n}  \left( \int_{\mathbb{R}^2} \int_{\mathbb{R}^2} \frac{d\mu(x) d\mu(y)}{1+\|x-y\|^2}  \right)^{-1} + \mbox{l.o.t.}
\end{align*}
It remains to analyze this integral. As before, we can assume that $\mu$ is concentrated at scale $r$ around a single point and use
$$ \int_{\mathbb{R}^2} \int_{\mathbb{R}^2} \frac{ d\mu(x) d\mu(y)}{1+\|x-y\|^2}   = 1 - \int_{\mathbb{R}^2} \int_{\mathbb{R}^2} \frac{ \|x-y\|^2}{1+\|x-y\|^2}    d\mu(x) d\mu(y).$$
The geometric series
$$ \frac{1}{1-x} = 1 + x + x^2 + \dots$$
leads to
\begin{align*}
 \left( \int_{\mathbb{R}^2} \int_{\mathbb{R}^2} \frac{d\mu(x) d\mu(y)}{1+\|x-y\|^2}  \right)^{-1} &= 1 +  \int_{\mathbb{R}^2} \int_{\mathbb{R}^2} \frac{ \|x-y\|^2}{1+\|x-y\|^2}    d\mu(x) d\mu(y) \\
 &+ \left( \int_{\mathbb{R}^2} \int_{\mathbb{R}^2} \frac{ \|x-y\|^2}{1+\|x-y\|^2}    d\mu(x) d\mu(y) \right)^2 + \mathcal{O}(r^6).
\end{align*}
Recalling
$$  \frac{\|x-y\|^2}{1+\|x-y\|^2} =  \|x-y\|^2 - \|x-y\|^4 + \mathcal{O}(\|x-y\|^6)$$
we can simplify these integrals as above and get
\begin{align*}
 \left( \int_{\mathbb{R}^2} \int_{\mathbb{R}^2} \frac{d\mu(x) d\mu(y)}{1+\|x-y\|^2}  \right)^{-1} &= 1 +  \int_{\mathbb{R}^2} \int_{\mathbb{R}^2}  \|x-y\|^2    d\mu(x) d\mu(y) \\
 &-  \int_{\mathbb{R}^2} \int_{\mathbb{R}^2}  \|x-y\|^4    d\mu(x) d\mu(y) \\
 &+ \left( \int_{\mathbb{R}^2} \int_{\mathbb{R}^2}  \|x-y\|^2  d\mu(x) d\mu(y) \right)^2 + \mathcal{O}(r^6).
\end{align*}
Altogether,
\begin{align*}
\mathbb{E} &= \mathbb{E}_2 -  p(n-1)  - p(n-1)  \int_{\mathbb{R}^2} \int_{\mathbb{R}^2}  \|x-y\|^2    d\mu(x) d\mu(y) + \mbox{l.o.t.} \\
&+  p(n-1)  \int_{\mathbb{R}^2} \int_{\mathbb{R}^2} \|x-y\|^4    d\mu(x) d\mu(y) - p(n-1) \left( \int_{\mathbb{R}^2} \int_{\mathbb{R}^2}  \|x-y\|^2  d\mu(x) d\mu(y) \right)^2 
\end{align*}
We note that, at the scale that we consider, only the first line will be relevant: the relevant terms in $\mathbb{E}, \mathbb{E}_2$ are at scale $\sim n^2 r^4$ which may be comparable to $\sim n r^2$ but for which $\sim n r^4$ is a lower order term.

\subsubsection{Conclusion.} This completes our Taylor expansion, we can now collect all the relevant terms for the Taylor expansion of the t-SNE energy with respect to a random $k-$regular graph up to leading order. For the expectation, we have our expansion for $\mathbb{E}_2$ and the correction term from the preceding section. After some simplification, we arrive at
\begin{align*}
\mathbb{E}~\mbox{t-SNE energy} &= pn(n-1) \log{(n^2)}  -  p(n-1)  \\
&+ p \frac{n^2 - 2}{2} \int_{\mathbb{R}^2} \int_{\mathbb{R}^2}   \|x-y\|^4 d\mu(x) d\mu(y) \\
&-  p \frac{(n-1)(n+2)}{2} \left( \int_{\mathbb{R}^2} \int_{\mathbb{R}^2} \|x-y\|^2 d\mu(x) d\mu(y)\right)^2 \\
& + p \int_{\mathbb{R}^2} \int_{\mathbb{R}^2}  \|x-y\|^2    d\mu(x) d\mu(y) + \mbox{l.o.t.}
\end{align*}
We see that there are some constants depending only on $n, p$, there are two terms with the same pre-factor that emulate the dominant Jensen-functional structure that we have already encountered above and there is a lower order perturbation. 
As for the variance, we recall that
$$
\mathbb{V} =  p(1-p)n^2\int_{\mathbb{R}^2} \int_{\mathbb{R}^2}  \|x - y\|^4 d \mu(x) d\mu(y) + \mbox{l.o.t.}.
$$
Having identified expectation and variance, the first approximation is naturally given by
$$ X \sim \mathbb{E}X \pm \sigma \sqrt{\mathbb{V} X},$$
where $\sigma$ is a random variable at scale $\sim 1$ (and, in many settings, one would expect it to be approximately Gaussian). 
This is exactly the ansatz that we chose for our functional. Ignoring the constants (which have no impact on the structure of the minimizer), dividing by $\sim pn^2 /2$ and absorbing some universal constants depending only on $p$ in the scaling of $\sigma$, we see that the ansatz leads to
\begin{align*}
J_{\sigma, \delta}(\mu) &= \int_{\mathbb{R}^2} \int_{\mathbb{R}^2}   \|x-y\|^4 d\mu(x) d\mu(y) - \left( \int_{\mathbb{R}^2} \int_{\mathbb{R}^2} \|x-y\|^2 d\mu(x) d\mu(y)\right)^2\\
&+  \frac{\sigma}{ \sqrt{p}} \delta \left( \int_{\mathbb{R}^2} \int_{\mathbb{R}^2}  \|x - y\|^4 d \mu(x) d\mu(y)\right)^{1/2},
\end{align*}
where $\delta \sim 1/n$.
We first note that this functional has a scaling symmetry.
If we replace the measure $\mu$ by the rescaled measure $\mu_{\lambda}$ (defined in the canonical way: $\mu_{\lambda}(A) = \mu(\lambda^{-1} A)),$ then we see that
\begin{align*}
\int_{\mathbb{R}^2} \int_{\mathbb{R}^2}   \|x-y\|^4 d\mu_{\lambda}(x) d\mu_{\lambda}(y)  &= \lambda^4 \int_{\mathbb{R}^2} \int_{\mathbb{R}^2}   \|x-y\|^4 d\mu(x) d\mu(y) \\
 \left( \int_{\mathbb{R}^2} \int_{\mathbb{R}^2} \|x-y\|^2 d\mu_{\lambda}(x) d\mu_{\lambda}(y)\right)^2 &= \lambda^4 \left( \int_{\mathbb{R}^2} \int_{\mathbb{R}^2} \|x-y\|^2 d\mu(x) d\mu(y)\right)^2 \\
% \int_{\mathbb{R}^2} \|x-y\|^2 d\mu_{\lambda}(x) d\mu_{\lambda}(y) &= \lambda^2 \int_{\mathbb{R}^2} \|x-y\|^2 d\mu(x) d\mu(y) \\
  \left(\int_{\mathbb{R}^2} \|x - y\|^4 d \mu_{\lambda}(x) d\mu_{\lambda}(y)\right)^{1/2} &= \lambda^2  \left(\int_{\mathbb{R}^2} \|x - y\|^4 d \mu(x) d\mu(y)\right)^{1/2}
\end{align*}
and therefore, for any $\lambda > 0$,
$$ J_{\sigma, \delta}(\mu_{\lambda}) \frac{1}{\lambda^4} = J_{\sigma, \delta  \lambda^{-2}}(\mu).$$
As
the number of points $n$ increases, $\delta$ decreases. This, however, does not fundamentally alter the functional, it merely changes the scale of
extremal configurations. We can thus without loss of generality assume that $\delta = 1$ and study the simplified functional
$J_{\sigma} := J_{\sigma, 1}.$

\subsection{Radial Solutions.} We will now analyze $J_{\sigma}$ for $\sigma$ fixed under the additional assumption that $\mu$ is radial. This is partially inspired by numerical results which seemed to result in radial configurations. It could be interesting to try to remove that assumption.
Assuming the measure $\mu$ to be radial, we will introduce $\nu$ as the measure on $\mathbb{R}_{\geq 0}$ such that for all $A \subset [0,\infty]$
$$ \nu(A) = \mu \left( \left\{x \in \mathbb{R}^2: \|x\| \in A \right\} \right).$$
This makes $\nu$ a probability measure on $[0,\infty]$.
We require the two basic integral identities
\begin{align*}
\frac{1}{2\pi r} \frac{1}{2\pi s} \int_{\|x\| = r} \int_{\|y\|=s} \|x-y\|^2 dx dy &= r^2 + s^2. \\
\frac{1}{2\pi r} \frac{1}{2\pi s} \int_{\|x\| = r} \int_{\|y\|=s} \|x-y\|^4 dx dy &= r^4 +4r^2 s^2 +  s^4.
\end{align*}
We then have, by switching to polar coordinates,
\begin{align*}
 \int_{\mathbb{R}^2} \int_{\mathbb{R}^2} \|x-y\|^4 d\mu(x) d\mu(y) &= \int_0^{\infty} \int_0^{\infty}  (r^4 +4r^2 s^2 +  s^4) d\nu(r) d\nu(s) \\
 &=  2 \int_0^{\infty} r^4 d\nu(r) + 4\left( \int_0^{\infty} r^2 d\nu(r) \right)^2.
 \end{align*}
Likewise, we have
\begin{align*}
 \int_{\mathbb{R}^2} \int_{\mathbb{R}^2}  \|x-y\|^2 d\mu(x) d\mu(y) &= \int_0^{\infty} \int_0^{\infty}  (r^2 +  s^2)  d\nu(r) d\nu(s) =  2 \int_0^{\infty} r^2 d\nu(r).
 \end{align*}
 Thus, for radial measures, the functional simplifies to (dividing without loss of generality by a factor of 2 for simplicity)
 \begin{align*}
J_{\sigma,1}(\nu) &=   \int_0^{\infty} r^4 d\nu(r)  +  \frac{\sigma}{\sqrt{p}}    \left(  \frac{1}{2} \int_0^{\infty} r^4 d\nu(r) + \left( \int_0^{\infty} r^2 d\nu(r) \right)^2 \right)^{1/2}.
\end{align*}
 At this point we can rewrite everything in terms of moments of a random variable $X$ that is distributed according to $X \sim \nu$ as
 $$ J_{\sigma,1}(\nu) = \mathbb{E} X^4  + \frac{\sigma}{\sqrt{p}}  \left(\frac{1}{2} \mathbb{E}X^4 + (\mathbb{E} X^2)^2\right)^{1/2}.$$
 We recall the Cauchy-Schwarz inequality
 $$0 \leq \mathbb{E} X^2 \leq \left(\mathbb{E} X^4\right)^{1/2}.$$
 We can thus reduce the problem to one in multivariable calculus: for all $0 \leq a \leq \sqrt{b}$, what can be said about the minimum of
 $$ f(a,b) = b  + c \cdot \sqrt{ \frac{b}{2} + a^2},$$
  where $c=\sigma \delta/\sqrt{p}$. Observe that
    $$ \frac{\partial f}{\partial b} = 1 + \frac{c}{4\sqrt{a^2 + b/2}}$$
which shows that for $c \geq 0$, the minimizer is given by the trivial solution where the entire mass is collected in a point $\nu = \delta_0$. Let us thus assume $c < 0$. Then 
  $$ \frac{\partial f}{\partial a} =  \frac{ac}{\sqrt{a^2 + b/2}} \leq 0$$
  and the functional decreases under increasing $a$. We thus want to have $a = \sqrt{b}$ which corresponds to the entire probability mass being collected in a single point. A simple computation shows that if $a = \sqrt{b}$, then the minimum of $f(\sqrt{b}, b)$ for $c<0$ is given by
  $ b_{*} = 3c^2/8$
  and thus the random variable is concentrated at distance $\sim \sqrt{|c|}$ form the origin. Recalling that we expect $\sigma \sim 1$, this corresponds to (for $\sigma < 0$) the functional $J_{\sigma, 1}$ assuming its minimum for a ring of radius $ \sim p^{-1/4}$.
  Recalling the scaling symmetry
 $ J_{\sigma, \delta}(\mu_{\lambda}) \lambda^{-4} = J_{\sigma, \delta  \lambda^{-2}}(\mu)$
   we thus the functional $ J_{\sigma, \delta}$ to assume its minimum for radius $\delta^{1/2} p^{-1/4}$.
Recalling $\delta = 1/n$, we arrive at the scaling of a ring forming at distance $\sim n^{-1/2} p^{-1/4}$. Finally, for a random $k-$regular graph, we have $k=p \cdot n$, this leads to $\sim k^{-1/4} n^{-1/4}$.

\section{Numerical Results}

We conclude with a discussion of some numerical experiments to test the assumptions on scaling that guided our derivation. Our underlying assumption is that there is simply no good way to embed a large random graph; the object is too high-dimensional. More precisely, we assumed that if we are given an embedding $\{y_1, \dots, y_n \} \subset \mathbb{R}^2$, then -- due to stochastic regularization -- the t-SNE energy $E$ will be essentially constant across all graphs sampled from a fixed model. We expect $E$ to be close to the expectation and that the typical deviation from the expectation is given by the variance. This section focuses on testing this hypothesis.

\begin{center}
\begin{figure}[h!]
\begin{tikzpicture}
\node at (-4,0) {\includegraphics[width=0.35\textwidth]{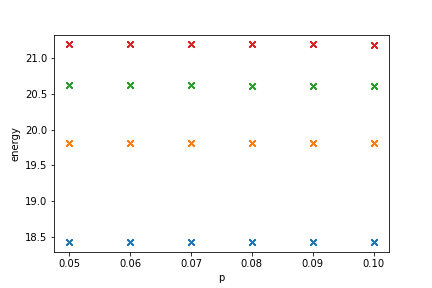}};
\node at (0,0) {\includegraphics[width=0.35\textwidth]{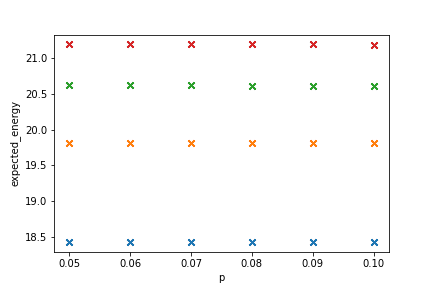}};
\node at (4,0) {\includegraphics[width=0.35\textwidth]{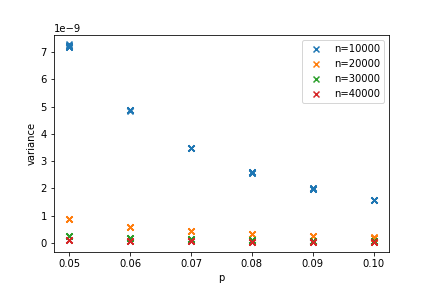}};
\end{tikzpicture}
\caption{Scatterplots of t-SNE energy (left), expected energy (center), and variance (right) of all trials of each parameter setting. We observe little variance between trials, which is unsurprising due to the stochastic regularity of the underlying graph model. The t-SNE energy is mostly explained by its expectation.}
\label{fig:stats}
\end{figure}
\end{center}

\subsection{Experiment Setup.}

We used the Networkx Python library to generate random regular graphs $G$, and we ran t-SNE using Kluger et. al's implementation \cite{george2} with our custom values for $P$. We normalized $P$ to a probability distribution, i.e. $p_{ij} \in \{0, 1/(2|E(G)|)\}$, though as discussed previously this has no effect on the minima of the t-SNE energy $E$. We ran t-SNE as follows: after a PCA initialization, we applied an early exaggeration factor of 12 for 250 iterations, and then finished with an additional 500 normal iterations. We checked that this was sufficient for the embedding to stabilize. With our chosen normalization, we calculated the t-SNE energy as
\begin{align*}
\mbox{t-SNE} &= \sum_{(i,j) \in E(G)} \frac{1}{|E(G)|}  \log \left( \sum_{i,j =1 \atop i \neq j}^n \frac{1}{1+ \|y_i - y_j\|^2} \right) \\
&+  \frac{1}{|E(G)|}\sum_{(i,j) \in E(G)} \log\left(1+\|y_i  - y_j\|^2\right) \\
\end{align*}
and the expectation and variance as
\begin{align*}
\mathbb{E} ~\mbox{t-SNE}&= \log \left( \sum_{i,j =1 \atop i \neq j}^n \frac{1}{1+ \|y_i - y_j\|^2} \right)  + \frac{p}{2 \cdot |E|} \sum_{i,j=1 \atop i \neq j}^{n}\log\left(1+\|y_i  - y_j\|^2\right) \\
\mathbb{V} ~\mbox{t-SNE} &= p(1-p) \sum_{i,j=1 \atop i \neq j}^n \frac{1}{4\cdot |E|^2}  \left(\log\left(1+\|y_i  - y_j\|^2\right)\right)^2.
\end{align*}
Since we work with random regular graphs, the variance comes from the second energy term only. We ran calculations for all graph parameter combinations with 
$$
n = 10\_000, 20\_000, 30\_000, 40\_000 \quad \text{and} \quad p = 0.05, 0.06, \ldots, 0.1.
$$
The graph degree is $k = n\cdot p$. We ran $10$ trials for each parameter setting.

\begin{center}
\begin{figure}[h!]
\begin{tikzpicture}
\node at (-5,0) {\includegraphics[width=0.3\textwidth]{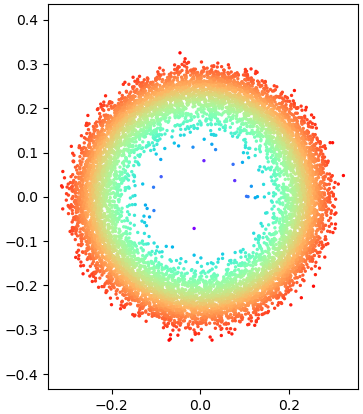}};
\node at (0,0) {\includegraphics[width=0.3\textwidth]{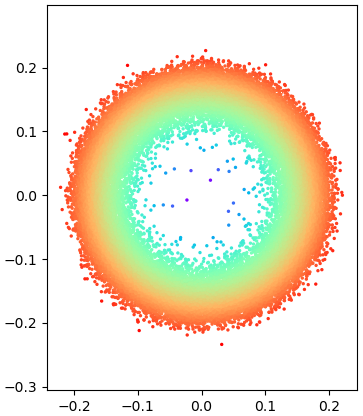}};
\end{tikzpicture}
\caption{t-SNE embedding of a random regular graph with $p=0.1$ and $n=10000$ (left) and $40000$ (right). The embedding diameter shrinks as $n$ increases due to the Jensen-like structure of the energy, which tends to concentrate the embedding at a point.}
\label{fig:ringn}
\end{figure}
\end{center}

\vspace{-20pt}

\subsection{Results.}

Our results confirm that the expected energy roughly equals the actual energy (Fig. \ref{fig:stats}). We also observe that in typical realizations, the expectation is many orders of magnitude larger than the variance. The calculations support our hypothesis that t-SNE works by minimizing the Jensen-like gap between the energy and its expected value, but injects some randomness to prevent the embedding from converging to a point mass. We observe how the Jensen structure tends to concentrate the embedding in Fig. \ref{fig:ringn}, which shows that the output diameter tends to decrease as $n$ increases.
We also hypothesize that the energy of the final embedding is described by:
$$
\mbox{t-SNE energy}(y_1, \dots, y_n) = \mathbb{E}(y_1, \dots, y_n)  + \sigma \sqrt{ \mathbb{V}(y_1, \dots, y_n)}
$$
where $\sigma$ takes values at scale $\sim 1$. We tested this conjectured relationship numerically by computing the actual energy, the expected energy and the variance and the solving for $\sigma$. The results are shown in Fig. \ref{fig:sigma} and suggest that this assumption is reasonable. Indeed, we emphasize that, due to the scaling by $2 \cdot |E|$, the expected energy is at scale $\sim 20$ while the variance is closer to $\sim 10^{-9}$. Having $\sigma \sim 1$ for these very different scales is a good indicator that our assumption on stochastic regularization is meaningful in this context. While it would be difficult to argue convincingly that $\sigma$ behaves like a Gaussian, it does seem as if it were roughly centered at 0 and has variance roughly $\sim 1$ (Fig. \ref{fig:sigma}).  

\vspace{-15pt}
\begin{center}
\begin{figure}[h!]
\begin{tikzpicture}
\node at (0,0) {\includegraphics[width=0.6\textwidth]{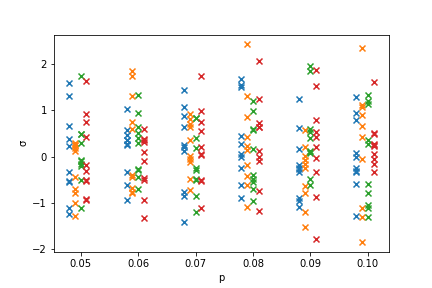}};
\node at (6,0) {\includegraphics[width=0.33\textwidth]{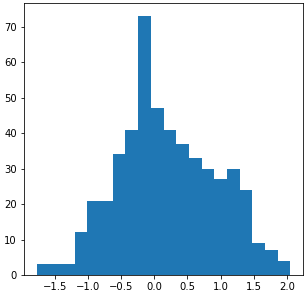}};
\end{tikzpicture}
\caption{Left: Scatterplots of $\sigma$ of all trials of each parameter setting. Right: Histogram of $\sigma$ for $n=10000$ and $p=0.1$.}
\label{fig:sigma}
\end{figure}
\end{center}

\end{document}